\documentclass[10pt,twocolumn,letterpaper]{article}
\pdfoutput=1

\usepackage[pagenumbers]{cvpr} % To force page numbers, e.g. for an arXiv version

\usepackage[accsupp]{axessibility} % Improves PDF readability for those with disabilities.

\usepackage[skip=0.333\baselineskip]{caption}
\usepackage{booktabs,tabularx,ragged2e}
\usepackage{multirow}
\usepackage{graphicx}
\usepackage{tabularx}
\usepackage[table, dvipsnames]{xcolor}

\newcolumntype{C}{>{\centering\arraybackslash}X}
\newcolumntype{Y}{>{\hsize=.7\hsize\RaggedRight\arraybackslash}X}
\newcolumntype{a}{>{\columncolor{lightgray}}Y}
\newcolumntype{C}[1]{>{\columncolor{lightgray}}p{#1}}
\usepackage{amsmath,amsfonts,bm}
\usepackage{amsthm}

\def\rva{{\mathbf{a}}}
\def\rvb{{\mathbf{b}}}

\def\rvn{{\mathbf{n}}}

\def\rvx{{\mathbf{x}}}
\def\rvy{{\mathbf{y}}}
\def\rvz{{\mathbf{z}}}

\def\rmI{{\mathbf{I}}}

\DeclareMathAlphabet{\mathsfit}{\encodingdefault}{\sfdefault}{m}{sl}
\SetMathAlphabet{\mathsfit}{bold}{\encodingdefault}{\sfdefault}{bx}{n}

\def\gN{{\mathcal{N}}}

\newcommand{\E}{\mathbb{E}}

\newcommand{\KL}{D_{\mathrm{KL}}}

\newtheorem{prop}{Proposition}
\newtheorem*{prop*}{Proposition}

\definecolor{cvprblue}{rgb}{0.21,0.49,0.74}
\usepackage[pagebackref,breaklinks,colorlinks,citecolor=cvprblue]{hyperref}

\def\paperID{4871} % *** Enter the Paper ID here
\def\confName{CVPR}
\def\confYear{2024}

\title{SeNM-VAE: Semi-Supervised Noise Modeling with Hierarchical Variational Autoencoder}

\author{
Dihan Zheng\textsuperscript{1}\footnotemark[1]
\qquad Yihang Zou\textsuperscript{1}\footnotemark[1] \qquad Xiaowen Zhang\textsuperscript{2} \qquad Chenglong Bao\textsuperscript{1,3,4}\footnotemark[2] \\
{\textsuperscript{1}Yau Mathematical Sciences Center, Tsinghua University, Beijing\quad\textsuperscript{2}Hisilicon, Shanghai} \\
{\textsuperscript{3}Yanqi Lake Beijing Institute of Mathematical Sciences and Applications, Beijing} \\
{\textsuperscript{4}State Key Laboratory of Membrane Biology, School of Life Sciences, Tsinghua University, Beijing} \\
{\textsuperscript{*}{\tt\normalsize Equal contribution} \quad \textsuperscript{$\dagger$}{\tt\normalsize Corresponding author}} \\
{\tt\small \{zhengdh19,zou-yh21\}@mails.tsinghua.edu.cn} \\
{\tt\small zhangxiaowen9@hisilicon.com \quad clbao@mail.tsinghua.edu.cn}
}

\begin{document}
\maketitle
\begin{abstract}
The data bottleneck has emerged as a fundamental challenge in learning based image restoration methods. Researchers have attempted to generate synthesized training data using paired or unpaired samples to address this challenge. This study proposes SeNM-VAE, a semi-supervised noise modeling method that leverages both paired and unpaired datasets to generate realistic degraded data. Our approach is based on modeling the conditional distribution of degraded and clean images with a specially designed graphical model. Under the variational inference framework, we develop an objective function for handling both paired and unpaired data. We employ our method to generate paired training samples for real-world image denoising and super-resolution tasks. Our approach excels in the quality of synthetic degraded images compared to other unpaired and paired noise modeling methods. Furthermore, our approach demonstrates remarkable performance in downstream image restoration tasks, even with limited paired data.
With more paired data, our method achieves the best performance on the SIDD dataset.
\end{abstract}    
\section{Introduction}
\label{sec:intro}

Image restoration is a fundamental and essential problem in image processing and computer vision, aiming to restore the underlying signal from its corrupted observation. Traditional methods employ the Maximum a Posteriori (MAP) framework, transforming the image restoration problem into an optimization problem. In these approaches, the objective function comprises a data fidelity term and a regularization term corresponding to the degradation and prior models, respectively. Over the years, the prior model has been extensively studied. Before the advent of deep learning, researchers utilized hand-crafted priors, such as sparsity~\cite{rudin1992nonlinear,perona1990scale}, non-local similarity~\cite{buades2005non,dabov2007color}, and low-rankness~\cite{dong2012nonlocal,gu2014weighted}. Recently, harnessing the power of deep neural networks has enabled achieving more accurate prior models through pre-trained generative models derived from a plethora of unlabeled clean signals~\cite{song2021solving,kawar2022denoising}. 
% Consequently, the challenge now lies in obtaining an accurate degradation model.

% Deep learning based methods for image restoration have achieved remarkable success, such as image denoising~\cite{zhang2017beyond,zhang2018ffdnet,guo2019toward}, and super-resolution~\cite{dong2014learning,dong2016accelerating,wang2018esrgan,lugmayr2020srflow}. These methods aim to learn an end-to-end restoration map using paired training data. Thanks to the powerful representational capabilities of deep neural networks, these methods typically yield better results than traditional methods. However, their effectiveness is contingent on the availability of high-quality paired training data.

Deep learning based methods have achieved remarkable success in image restoration tasks, such as image denoising~\cite{zhang2017beyond,zhang2018ffdnet,guo2019toward} and super-resolution (SR)~\cite{dong2014learning,dong2016accelerating,wang2018esrgan,lugmayr2020srflow}. These methods aim to learn an end-to-end mapping for restoration using paired training data. Owing to the powerful representational capabilities of deep neural networks, these methods typically outperform traditional approaches. However, their effectiveness is contingent upon the availability of high-quality paired training data.

Collecting training data poses its challenges. First, real-world degradation is highly complex due to the intricate camera image signal processing (ISP) pipeline~\cite{hasinoff2010noise,abdelhamed2019noise,gow2007comprehensive}, rendering the simulation process challenging. Another approach entails manual collection, where clean and degraded pairs are obtained through long and short exposure~\cite{brummer2019natural} or using statistical methods~\cite{abdelhamed2018high,plotz2017benchmarking}. Nevertheless, these approaches inevitably suffer from the misalignment problem between clean and degraded images~\cite{wei2019single}, making the process expensive and time-consuming.

We conclude that a key problem in real-world image restoration is obtaining an accurate degradation model. With the degradation model, one can tackle the restoration problem either via a classical optimization based method with a pre-defined prior model~\cite{zheng2020unsupervised,cheng2023score} or a supervised learning based method with synthesized training data~\cite{zheng2022learn,wolf2021deflow}. Accordingly, we investigate a scenario where a limited amount of paired data and a large amount of unpaired data are available, referred to as a semi-supervised dataset. Our approach involves learning the unknown degradation model from this semi-supervised training dataset and synthesizing more paired data using this degradation model. We then use existing image restoration networks to learn a supervised image restoration model from the synthesized data. To obtain the degradation model, we design a graphical model that characterizes the relationship between the noisy image $\rvy$ and the clean image $\rvx$. We introduce two latent variables, $\rvz$, and $\rvz_\rvn$, representing the image content and degradation information, respectively. Furthermore, we assume that $\rvx$ is generated by $\rvz$, and $\rvy$ is generated by $\rvz$ and $\rvz_\rvn$. Using the idea from VAE~\cite{kingma2013auto}, we approximate the conditional distribution $p(\rvy|\rvx)$ with encoding and decoding processes. To effectively utilize the semi-supervised dataset, we employ a mixed inference model for $q(\rvz |\rvx,\rvy)$ to further decompose the objective function for paired and unpaired datasets. We apply our SeNM-VAE model to learn real-world image degradation.
% Experimental results show that the proposed SeNM-VAE model exhibits promising performance in noise modeling and performs comparably to the supervised learning method trained with fully paired data.
Experimental results demonstrate that the proposed SeNM-VAE model exhibits promising performance in noise modeling, achieving comparable results to supervised learning methods trained with fully paired data.
% outperforms other unpaired noise modeling methods and is close to the supervised learning method trained with fully paired data. 
Moreover, we achieve the best performance by finetuning an existing denoising network on the SIDD benchmark. Our main contributions are summarized as follows.
\begin{itemize}
    \item Leveraging limited paired data and abundant unpaired data, we propose SeNM-VAE to obtain an effective model for simulating the degradation process. This is important for generating high-quality training samples for real-world image restoration tasks when obtaining training samples is difficult. Using the variational inference method, SeNM-VAE is based on a specially designed graphical model and a hierarchical structure with multi-layer latent variables.
    
    \item Experimental results on the real-world noise modeling and downstream applications, such as image denoising and SR, validate the advantages of the proposed SeNM-VAE, and we achieve the best performance on the SIDD benchmark.
\end{itemize}

\section{Related work}
\label{sec:related_work}

{\noindent \bf Semi-supervised image restoration.} To alleviate the challenges of acquiring paired data for image restoration tasks, researchers have been investigating semi-supervised techniques, including image dehazing~\cite{li2019semi}, deraining~\cite{wei2019semi,cui2022semi,jiang2023lightweight,wei2021semi}, and low-light image enhancement~\cite{malik2023semi}. These approaches typically rely on image priors, such as Total Variation (TV)~\cite{rudin1992nonlinear} and the dark channel prior~\cite{he2010single}, to formulate a loss function for unlabeled datasets. Some recent works, such as~\cite{wei2021semi}, have employed CycleGAN~\cite{zhu2017unpaired} loss for unpaired datasets. However, these methods are often designed heuristically and lack theoretical rigor. In contrast, our approach is based on a specially designed graphical model, and the loss function is derived through variational inference, enhancing our method's interpretability.

{\noindent \bf Deep degradation modeling.} Due to the limitations of Gaussian noise in capturing the signal-dependence of real-world noise~\cite{plotz2017benchmarking,zheng2020unsupervised}, researchers are exploring data-driven approaches that utilize either normalizing flow~\cite{abdelhamed2019noise,kousha2022modeling} or GAN~\cite{yue2020dual} to generate realistic noisy images with paired training data. Furthermore, some studies have employed unpaired data to learn the unknown degradation process, which mostly utilizes the GAN model and techniques such as cycle-consistency~\cite{bulat2018learn,lugmayr2019unsupervised} proposed in~\cite{zhu2017unpaired} and domain adversarial~\cite{fritsche2019frequency,wei2021unsupervised}. Other unpaired degradation modeling methods include those based on Flow and VAE~\cite{wolf2021deflow,zheng2022learn}. In contrast, we propose a semi-supervised degradation modeling approach for real-world image restoration that leverages both paired and unpaired datasets. % To the best of our knowledge, our work represents the first attempt to model the degradation for real-world image denoising in the semi-supervised learning setting.
\section{Our methodology}

In this section, we present our semi-supervised noise modeling method. Formally, our goal is to estimate the conditional distribution $p(\rvy|\rvx)$ with one paired dataset $\{(\rvx_i, \rvy_i)\}_{i=1}^{N_p}$ and two unpaired datasets $\{\rvx_i\}_{i=1}^{N_s}$, $\{\rvy_i\}_{i=1}^{N_t}$, where $N_p$, $N_s$, and $N_t$ denote the number of paired, source, and target training samples, respectively. Subsequently, we can generate paired training samples using the data from the source domain. In general, we can learn a conditional generative model to sample from $p(\rvy|\rvx)$ with paired samples $\{(\rvx_i, \rvy_i)\}$, whereas those conventional generative models are incapable of utilizing unpaired samples. Furthermore, if $N_p$ is small, achieving an accurate generative model becomes challenging. In this work, we propose a model that exploits the information of unpaired data with the assistance of the provided paired samples.

\begin{figure}
	\centering
	\includegraphics[width=1\linewidth]{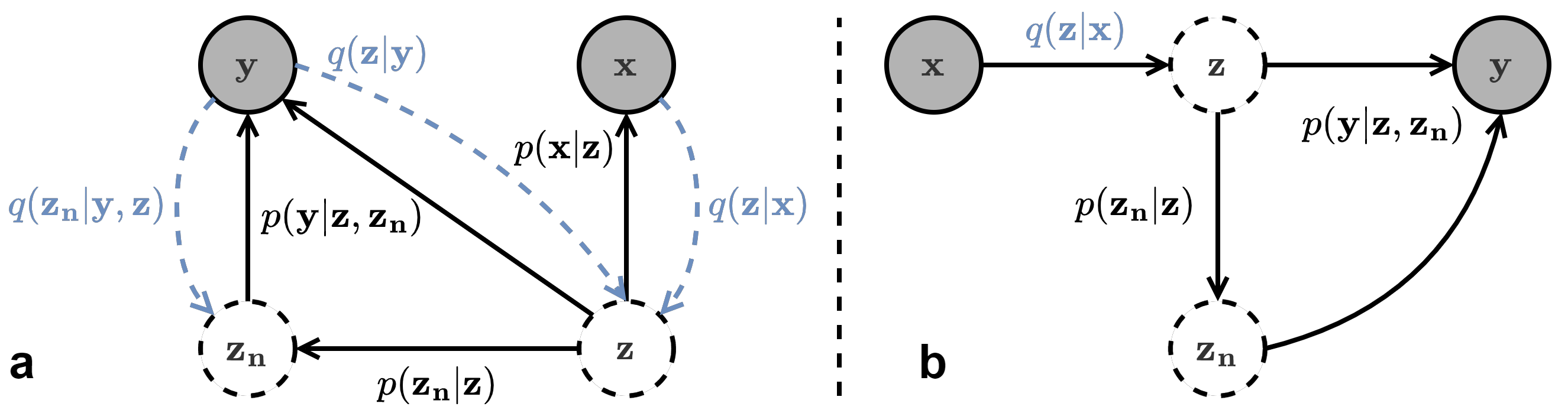}
	\caption{\textbf{a}: generation process for $(\rvx, \rvy)$ and the corresponding inference model. \textbf{b}: degradation generation procedure.}
	\label{graphical_model_noise_generation}
\end{figure}

\begin{figure}
	\centering
	\includegraphics[width=1\linewidth]{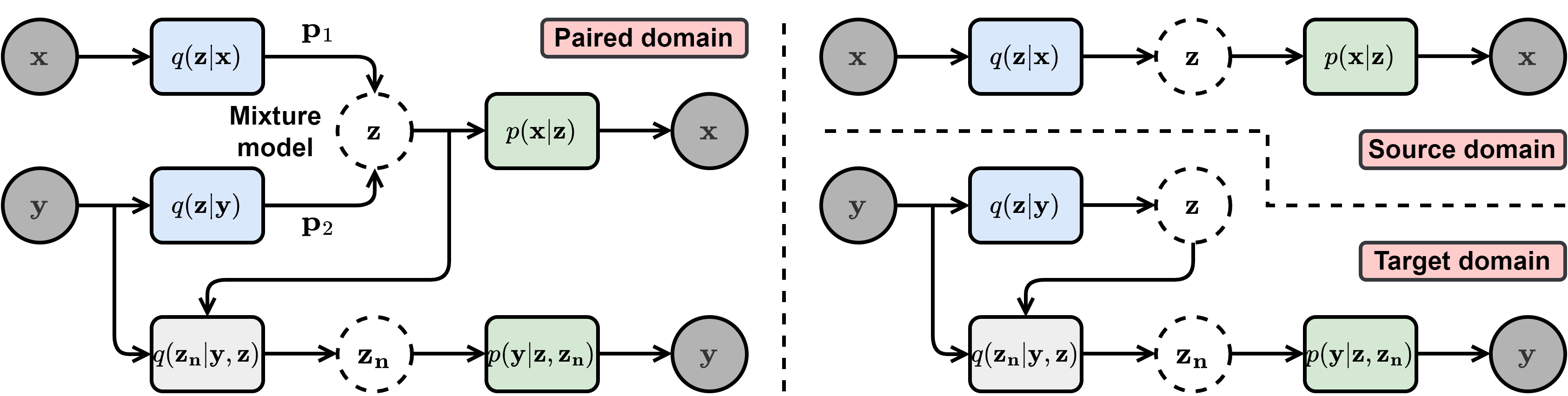}
	\caption{Data flow of the proposed semi-supervised noise modeling method that models three kinds of data: paired domain (degraded-clean image pairs), source domain (only clean images), and target domain (only degraded images).}
	\label{dataflow}
\end{figure}

 \begin{figure*}
    \centering
    \includegraphics[width=1\linewidth]{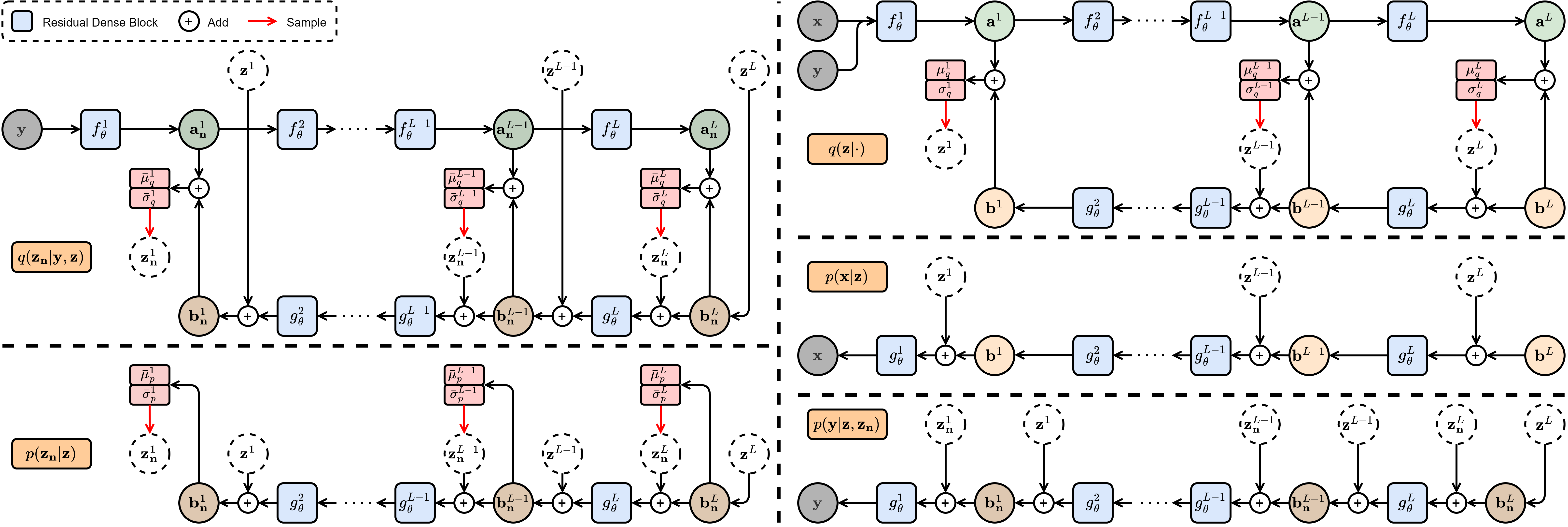}
    \caption{The hierarchical structure and network architecture.}
    \label{network_arch}
 \end{figure*}

\subsection{Proposed model}
To estimate the conditional density function $p(\rvy|\rvx)$, the traditional Maximal Likelihood (ML) estimation is to maximize the conditional log-likelihood function: $\max_{\theta} \E_{p(\rvx,\rvy)} \log p_{\theta}(\rvy|\rvx)$, where $\theta$ denotes the model parameter. Parameterizing $p(\rvy|\rvx)$ in high-dimensional spaces directly is usually difficult. Thus, we consider the latent variable model. In particular, we assume there are two latent variables, $\rvz$ and $\rvz_\rvn$, which encode the image content and degradation information, respectively. 
% Ideally, we may assume that $\rvz$ and $\rvz_\rvn$ are independent, and the joint prior distribution can be factorized as $p^{\bot}(\rvz,\rvz_\rvn) = p(\rvz)p(\rvz_\rvn)$. However, in practice, the degradation is signal dependent, and thus $\rvz$, $\rvz_\rvn$ can be entangled with each other in the latent space. 
Since practical degradation is signal-dependent, $\rvz$ and $\rvz_\rvn$ may be entangled with each other in the latent space. 
To account for this entanglement, we employ the factorization $p(\rvz,\rvz_\rvn) = p(\rvz)p(\rvz_\rvn|\rvz)$, where we assume that $\rvz_\rvn$ is generated from $\rvz$. Moreover, we assume $\rvx$ is generated by $\rvz$, and $\rvy$ is generated by $\rvz$ and $\rvz_\rvn$, see Figure~\ref{graphical_model_noise_generation}a for the generative process of $(\rvx, \rvy)$. By introducing an inference model $q(\rvz, \rvz_\rvn | \rvx, \rvy)$, we can decompose the conditional log-likelihood as 
\begin{equation}\label{vi_decomposition}
	\begin{aligned}
		\log p(\rvy|\rvx) &= \E_{q(\rvz, \rvz_\rvn | \rvx, \rvy)} \log \frac{p(\rvy, \rvz, \rvz_\rvn | \rvx)}{q(\rvz, \rvz_\rvn | \rvx, \rvy)} \\
		& + \KL (q(\rvz, \rvz_\rvn | \rvx, \rvy) \| p(\rvz, \rvz_\rvn | \rvx, \rvy)),
	\end{aligned}
\end{equation}
where $\KL$ denotes the Kullback–Leibler (KL) divergence, and the expectation term in \eqref{vi_decomposition} is called the conditional Evidence Lower BOund (cELBO)~\cite{sohn2015learning}. Due to the intractability of the original log-likelihood function, we choose to maximize the cELBO for density estimation. From the graphical model in Figure~\ref{graphical_model_noise_generation}a, we have
\begin{equation}
	\begin{aligned}
		& p(\rvy, \rvz, \rvz_\rvn|\rvx) = p(\rvz|\rvx) p(\rvz_\rvn|\rvz) p(\rvy | \rvz,\rvz_\rvn), \\
		& p(\rvz, \rvz_\rvn | \rvx, \rvy) = p(\rvz | \rvx, \rvy) p(\rvz_\rvn | \rvy, \rvz).
	\end{aligned}
\end{equation}
To match the factorization of $p(\rvz, \rvz_\rvn | \rvx, \rvy)$, we choose 
\begin{equation}
	q(\rvz, \rvz_\rvn | \rvx, \rvy) = q(\rvz | \rvx, \rvy) q(\rvz_\rvn | \rvy, \rvz),
\end{equation}
then the cELBO becomes
\begin{equation}\label{cELBO}
	\begin{aligned}
		&\E_{q(\rvz | \rvx, \rvy)q(\rvz_\rvn|\rvy,\rvz)} \log p(\rvy|\rvz,\rvz_\rvn) - \KL(q(\rvz|\rvx,\rvy) \| p(\rvz|\rvx)) \\
        &- \E_{q(\rvz|\rvx,\rvy)}\KL(q(\rvz_\rvn|\rvy,\rvz)\|p(\rvz_\rvn|\rvz)).
	\end{aligned}
\end{equation}
Please refer to the Appendix for the detailed derivation process. However, the inference model $q(\rvz | \rvx, \rvy)$ requires paired training data, which hinders the decomposition of the above cELBO to utilize unpaired data. One approach to overcome this limitation and further decouple $q(\rvz | \rvx, \rvy)$ is to use a mixture model~\cite{shi2019variational}. Specifically, we define $q(\rvz|\rvx,\rvy)$ as a linear combination of two mixture components $q(\rvz | \rvx)$ and $q(\rvz | \rvy)$, \ie,
\begin{equation}\label{mix_model}
	q(\rvz|\rvx,\rvy) = p_1 q(\rvz | \rvx) + p_2 q(\rvz | \rvy).
\end{equation}
where $p_1$ and $p_2$ are the mixture weights, and we choose $p_1 = p_2 = 0.5$ in our method. However, this formulation leads to the absence of a closed-form solution for the KL divergence between $q(\rvz|\rvx,\rvy)$ and the conditional prior distribution $p(\rvz|\rvx)$ in the cELBO. Fortunately, we have the following proposition:
\begin{prop}\label{kl_bound}
Let $q(\rvz|\rvx,\rvy)$ be a mixture model of $q(\rvz|\rvx)$ and $q(\rvz|\rvy)$, as described in~\eqref{mix_model}, then:
\begin{equation}\label{kl_ineq}
\begin{aligned}
    \KL(q(\rvz|\rvx,\rvy) \| p(\rvz|\rvx)) \leq & p_1 \KL (q(\rvz|\rvx) \| p(\rvz|\rvx)) \\
    +& p_2 \KL (q(\rvz|\rvy) \| p(\rvz|\rvx)).
\end{aligned}
\end{equation}
Moreover, suppose that $q(\rvz|\rvx)=p(\rvz|\rvx)$ by sharing the same neural network. Then:
\begin{equation}
\KL(q(\rvz|\rvx,\rvy) \| p(\rvz|\rvx)) \leq p_2 \KL (q(\rvz|\rvy) \| q(\rvz|\rvx)).
\end{equation}
\end{prop}
See the supplemental material file for the proof. Utilizing proposition~\ref{kl_bound}, we obtain a lower bound for \eqref{cELBO}:
\begin{equation}\label{cELBO_lb}
	\begin{aligned}
		&\E_{q(\rvz | \rvx, \rvy)q(\rvz_\rvn|\rvy,\rvz)} \log p(\rvy|\rvz,\rvz_\rvn) - p_2 \KL (q(\rvz|\rvy) \| q(\rvz|\rvx)) \\
        &- \E_{q(\rvz|\rvx,\rvy)}\KL(q(\rvz_\rvn|\rvy,\rvz)\|p(\rvz_\rvn|\rvz)).
	\end{aligned}
\end{equation}
% \begin{rem}
% The KL term $\KL(q(\rvz|\rvy)\|q(\rvz|\rvx))$ can be viewed as an explicit regularization that enforces the inference invariant condition suggested in \cite{zheng2022learn}. This condition asserts that, given any pair of clean and degraded images $\rvx$ and $\rvy$, they share the same latent distribution of $\rvz$:
% \begin{equation}
% 	q(\rvz | \rvx) = q(\rvz | \rvy), \quad \forall (\rvx, \rvy) \sim p(\rvx, \rvy).
% \end{equation}
% \end{rem}
Furthermore, we introduce a reconstruction term, $\E_{q(\rvz|\rvx,\rvy)} \log p(\rvx|\rvz)$, for the source domain data. This reconstruction term serves two purposes. Firstly, it facilitates the utilization of the source domain data. Secondly, it aids in regularizing the inference model, $q(\rvz|\rvx,\rvy)$, ensuring that the image content variable, $\rvz$, encapsulates all essential information from the clean image, $\rvx$, such that it can be faithfully reconstructed through $p(\rvx|\rvz)$. The objective function for our method is defined as:
\begin{equation}\label{loss}
\begin{aligned}
\text{Loss} &= \E_{q(\rvz|\rvx,\rvy)}\KL(q(\rvz_\rvn|\rvy,\rvz)\|p(\rvz_\rvn|\rvz)) \\
& + \lambda \KL (q(\rvz|\rvy) \| q(\rvz|\rvx)) - \E_{q(\rvz|\rvx,\rvy)} \log p(\rvx|\rvz)\\
& - \E_{q(\rvz | \rvx, \rvy)q(\rvz_\rvn|\rvy,\rvz)} \log p(\rvy|\rvz,\rvz_\rvn).
\end{aligned}
\end{equation}
In this formulation, we draw inspiration from the $\beta$-VAE framework~\cite{higgins2016beta}, where we introduce a weight parameter $\lambda$ in front of the term $\KL (q(\rvz|\rvy) \| q(\rvz|\rvx))$ to effectively balance the two KL divergence terms in~\eqref{loss}. Consequently, we can decompose~\eqref{loss} as
\begin{equation}
    \text{Loss} = \text{Loss}_p + \text{Loss}_s + \text{Loss}_t,
\end{equation}
where
\begin{equation}\label{loss_sep}
	\begin{aligned}
		\text{Loss}_{p} :=& - p_1 \E_{q(\rvz | \rvx)q(\rvz_\rvn|\rvy,\rvz)} \log p(\rvy|\rvz,\rvz_\rvn) \\
		+& p_1 \E_{q(\rvz|\rvx)}\KL(q(\rvz_\rvn|\rvy,\rvz)\|p(\rvz_\rvn|\rvz)) \\
		-& p_2 \E_{q(\rvz|\rvy)} \log p(\rvx | \rvz) + \lambda \KL (q(\rvz|\rvy) \| q(\rvz|\rvx)), \\
		\text{Loss}_{s} :=& - p_1 \E_{q(\rvz|\rvx)} \log p(\rvx | \rvz), \\
		\text{Loss}_{t} :=& - p_2 \E_{q(\rvz | \rvy)q(\rvz_\rvn|\rvy,\rvz)} \log p(\rvy|\rvz,\rvz_\rvn) \\
		&+ p_2 \E_{q(\rvz|\rvy)}\KL(q(\rvz_\rvn|\rvy,\rvz)\|p(\rvz_\rvn|\rvz)).
	\end{aligned}
\end{equation}
% It is clear that $\text{Loss}_{p}$, $\text{Loss}_{s}$, and $\text{Loss}_{t}$ are related to data from the paired, source, and target domain, respectively. Consequently, we can compute the loss function in~\eqref{loss} with both paired and unpaired datasets. The data flow of our method for both paired and unpaired domains is depicted in Figure~\ref{dataflow}.
It is clear that $\text{Loss}_{p}$, $\text{Loss}_{s}$, and $\text{Loss}_{t}$ correspond to data from the paired, source, and target domains, respectively. Therefore, the loss function described in~\eqref{loss} can be computed with both paired and unpaired datasets. The data flow for each domain within our approach is depicted in Figure~\ref{dataflow}.

\subsection{Model settings}

{\noindent \bf Degradation generation.} Upon completion of the training process, it becomes feasible to synthesize degraded data from a clean input. Given a clean image $\rvx$, we can generate the corresponding degraded image from $p(\rvy | \rvx)$ using ancestral sampling. This process involves first deriving the image content latent variable $\rvz$ from $q(\rvz|\rvx)$, followed by sampling the degradation latent variable $\rvz_\rvn$ from $p(\rvz_\rvn|\rvz)$. Subsequently, the corresponding degraded image $\rvy$ is generated from $p(\rvy|\rvz, \rvz_\rvn)$, as depicted in Figure~\ref{graphical_model_noise_generation}b.

{\noindent \bf Hierarchical structure.} To enhance the quality of our generative performance, we employ a hierarchical VAE architecture, as proposed in previous studies~\cite{vahdat2020nvae,child2020very}. Specifically, we assume that our latent variable $\rvz$ and $\rvz_\rvn$ is composed of $L$ layers:
\begin{equation}
	\rvz = (\rvz^1,\dots,\rvz^L), \quad \rvz_\rvn = (\rvz_\rvn^1,\dots,\rvz_\rvn^L).
\end{equation}
Using the chain rule, the probabilistic distribution in~\eqref{loss} can be decomposed as
\begin{equation}\label{decomp}
    \begin{aligned}
    & q(\rvz|\rvx) = \prod_{l=1}^{L} q(\rvz^l|\rvx,\rvz^{>l}), q(\rvz_\rvn|\rvy,\rvz) = \prod_{l=1}^{L} q(\rvz_\rvn^l |\rvy,\rvz,\rvz_\rvn^{>l}), \\    
    & q(\rvz|\rvy) = \prod_{l=1}^{L} q(\rvz^1|\rvy,\rvz^{>1}), p(\rvz_\rvn|\rvz) = \prod_{l=1}^{L} p(\rvz_\rvn^{l} | \rvz,\rvz_\rvn^{>l}). \\ 
    \end{aligned}
\end{equation}
where $\rvz_\rvn^{>l} = (\rvz_\rvn^{l+1},\dots,\rvz_\rvn^{L})$. Then, the KL divergences in~\eqref{loss} can be factorized as:
\begin{align}
    &\KL(q(\rvz_\rvn|\rvy,\rvz)\|p(\rvz_\rvn|\rvz)) =  \KL(q(\rvz_\rvn^L|\rvy,\rvz)\|p(\rvz_\rvn^L|\rvz)) \nonumber \\
    &+ \sum_{l=1}^{L-1} \E_{q(\rvz_\rvn^{>l}|\rvy,\rvz)} \left[\KL(q(\rvz_\rvn^l|\rvy,\rvz,\rvz_\rvn^{>l}) \| p(\rvz_\rvn^l|\rvz, \rvz_\rvn^{>l})) \right], \nonumber \\
    &\KL(q(\rvz|\rvy)\|q(\rvz|\rvx)) =  \KL(q(\rvz^L|\rvy)\|q(\rvz^L|\rvx)) \nonumber \\
    &+ \sum_{l=1}^{L-1} \E_{q(\rvz^{>l}|\rvy)} \left[\KL(q(\rvz^l|\rvy,\rvz^{>l}) \| q(\rvz^l|\rvx, \rvz^{>l})) \right], \nonumber
\end{align}
where the conditional distributions $q(\rvz_\rvn^l|\rvz_\rvn^{>l},\rvy,\rvz)$, $p(\rvz_\rvn^l|\rvz, \rvz_\rvn^{>l})$, $q(\rvz^l|\rvy,\rvz^{>l})$, and $q(\rvz^l|\rvx, \rvz^{>l})$ are chosen to be Gaussian distributions, allowing us to calculate the KL divergence in a closed-form expression. 

{\noindent \bf Model architecture.} For the inference model $q(\rvz|\rvx)$, we choose
% The inference model $q(\rvz|\rvx)$ is chosen as
\begin{equation}
 	q(\rvz^l|\rvx,\rvz^{>l}) = \gN(\mu_{q}^l(\rva^l, \rvb^l), \sigma_{q}^l(\rva^l, \rvb^l)), \quad l=1,\dots,L,
 \end{equation}
where $\rva^l$ and $\rvb^l$ are encoding and decoding features in $l$-th layer, respectively, and $\mu^l_{q}$ and $\sigma^l_{q}$ are networks that convert $(\rva^l, \rvb^l)$ to the parameters of a Gaussian distribution. The encoding features $\{\rva^l\}_{l=1}^L$ are recursively obtained through
\begin{equation}
 	\rva^{1} = f_{\theta}^{1}(\rvx), \quad \rva^{l} = f_{\theta}^{l}(\rva^{l-1}), \quad l=2,\dots,L,
\end{equation}
where $f_{\theta}^{l}$ represents the basic block in $l$-th layer. The decoding feature $\rvb^l$ can be obtained through:
\begin{equation}
 	\rvb^{L}=\mathbf{0}, \quad \rvb^{l-1} = g_{\theta}^l(\rvz^{l}, \rvb^l), \quad l = 2,\dots,L,
\end{equation}
where $\rvz^{l}$ is sampled from $\gN(\mu_{q}^l(\rva^l, \rvb^l), \sigma_{q}^L(\rva^l, \rvb^l))$, and $g_{\theta}^l$ is the basic block in $l$-th decoding layer. We choose the structure of $q(\rvz|\rvy)$ to be the same as $q(\rvz|\rvx)$.

For the inference model $q(\rvz_\rvn|\rvy,\rvz)$, we assume that the degradation latent variable $\rvz_\rvn$ is distributed as follows:
\begin{equation}\label{qzn_dis}
 	q(\rvz_\rvn^l|\rvy,\rvz,\rvz_\rvn^{>l}) = \gN(\bar{\mu}_{q}^l(\rva_{\rvn}^l, \rvb_{\rvn}^l), \bar{\sigma}_{q}^l(\rva_\rvn^l, \rvb_\rvn^l)), l=1,\dots,L,
 \end{equation}
where $\rva_{\rvn}^l$ and $\rvb_\rvn^l$ are encoding and decoding features, respectively, and $\bar{\mu}_{q}^l$ and $\bar{\sigma}_{q}^l$ are Gaussian parameterization networks in the $l$-th layer. The encoding feature $\rva_\rvn^l$ is computed recursively as follows:
\begin{equation}
 	\rva^{1}_\rvn = f_{\theta}^{1}(\rvy), \quad \rva^{l}_\rvn = f_{\theta}^{l}(\rva^{l-1}_\rvn), \quad l=2,\dots,L,
 \end{equation}
and the decoding feature is obtained through
\begin{equation}\label{dec_fea_zn}
 \rvb^{L}_\rvn = \rvz^{L}, \quad \rvb^{l-1}_\rvn = \rvz^{l-1} + g_{\theta}^l(\rvz^{l}_\rvn, \rvb^l_\rvn), \quad l = 2,\dots,L,
\end{equation}
where $\rvz^{l}_\rvn$ is sampled from~\eqref{qzn_dis}. 

For the conditional prior distribution $p(\rvz_\rvn|\rvz)$, we employ the same architecture as in $q(\rvz_\rvn|\rvy,\rvz)$, and assume
\begin{equation}
    \begin{aligned}
	 		p(\rvz_\rvn^l|\rvz, \rvz_\rvn^{>l}) = \gN(\bar{\mu}_{p}^l(\rvb_{\rvn}^l), \bar{\sigma}_{p}^l(\rvb_{\rvn}^l)), \quad l = 1,\dots,L,
	 \end{aligned}
\end{equation}
where the decoding feature $\rvb_{\rvn}^l$ is derived from~\eqref{dec_fea_zn}, $\bar{\mu}_p^l$ and $\bar{\sigma}_p^l$ are Gaussian parameterization networks.

In the case of the generative models, we choose
 \begin{equation}
	 p(\rvx|\rvz) = \gN(g_\theta^1(\rvz^1,\rvb^1), \rmI), p(\rvy|\rvz,\rvz_\rvn) = \gN(g_\theta^1(\rvz_\rvn^1,\rvb_\rvn^1), \rmI).
\end{equation}

We adopt the Residual Dense Block~\cite{zhang2018residual} (RDB) as our basic block for $f_{\theta}^l$ and $g_\theta^l$, see Figure~\ref{network_arch}.

{\noindent \bf Degradation level controllable generation.} The degradation generation model should be capable of producing images with varying degradation levels, enabling the generation of images with different degradation levels using a single clean input. Therefore, we leverage training data from the paired domain to train a degradation level prediction network and then utilize this network to estimate the degradation level of images within the target domain. During training, we concatenate the latent variable $\rvz_\rvn$ sampled from $q(\rvz_\rvn|\rvy,\rvz)$ with its corresponding degradation level. In the generation stage, we concatenate the specified degradation level to $\rvz_\rvn$ sampled from $p(\rvz_\rvn|\rvz)$ to enable conditional image generation.
\section{Experiment and results}

We first evaluate the performance of our model in real-world noise modeling tasks and then validate the downstream performance on image denoising and SR tasks. In particular, we utilize our model to learn the unknown degradation process, and then we generate synthetic degraded images from the source domain to augment the original paired training samples. Finally, we apply an off-the-shelf supervised learning network to derive a restoration model from both the synthetic dataset and the original paired data.

\subsection{Datasets}
% {\noindent \bf SIDD}: The smartphone image denoising dataset (SIDD)~\cite{abdelhamed2018high} offers a collection of 30,000 noisy images captured by five representative smartphone cameras across 10 different scenes under varying lighting conditions, along with their corresponding ground truth images. Here, we utilize the SIDD-Medium dataset, which includes 320 paired clean and noisy images. For each image in the dataset, we randomly cropped 300 patches of size $256 \times 256$, yielding a total of 96,000 paired data. To establish a semi-supervised dataset, we randomly select 0.01\% (10), 0.1\% (96), and 1\% (960) paired samples from the cropped SIDD-Medium dataset to serve as the paired domain sub-dataset. These 96,000 paired images are then divided into two subsets, each containing 48,000 paired images. We utilize the clean images from the first subset as the source domain sub-dataset and the noisy images from the second subset as the target domain sub-dataset. Consequently, the source and target domains contain 48,000 unpaired clean and noisy images. The SIDD dataset also includes validation and benchmark datasets, each containing 1,280 image blocks of size $256 \times 256$.
{\noindent \bf SIDD}: The smartphone image denoising dataset (SIDD)~\cite{abdelhamed2018high} offers a collection of 30,000 noisy images captured by five representative smartphone cameras across ten diverse scenes under varying lighting conditions, alongside their corresponding ground truth images. Here, we utilize the SIDD-Medium dataset, which includes 320 paired clean and noisy images. For each image in the dataset, we randomly crop 300 patches of size $256 \times 256$, yielding a total of 96,000 paired data. To establish a semi-supervised dataset, we randomly select 0.01\% (10), 0.1\% (96), and 1\% (960) paired samples from the cropped SIDD-Medium dataset, serving as the paired domain. These 96,000 paired images are then divided into two subsets, each containing 48,000 paired images. We utilize the clean images from the first subset as the source domain and the noisy images from the second subset as the target domain, resulting in 48,000 unpaired clean and noisy images in each domain. The SIDD dataset also includes validation and benchmark datasets, each containing 1,280 image blocks of size $256 \times 256$.

% {\noindent \bf DND}: The Darmstadt Noise Dataset (DND)~\cite{plotz2017benchmarking} comprises a benchmark dataset that contains 1,000 image blocks of size $512 \times 512$ extracted from 50 real-world noisy images acquired by four commercial cameras. We evaluate the models trained with the SIDD dataset on this benchmark directly.
{\noindent \bf DND}: The Darmstadt Noise Dataset (DND)~\cite{plotz2017benchmarking} comprises a benchmark dataset containing 1,000 image blocks of size $512 \times 512$ extracted from 50 real-world noisy images obtained from four commercial cameras. We directly assess models trained with the SIDD dataset on this benchmark.

{\noindent \bf AIM19}: Track 2 of the AIM 2019 real-world SR challenge~\cite{lugmayr2019aim} provides a dataset of unpaired degraded and clean images. The degraded images are synthesized with an unknown combination of noise and compression. The challenge also provides a validation set of 100 paired images. To construct a semi-supervised dataset, we select the first 10 paired images from the validation dataset to serve as the paired domain and leverage the originally provided unpaired dataset as source and target domains. Performance evaluation is then conducted on the remaining 90 images within the validation set.

{\noindent \bf NTIRE20}: Track 1 of NTIRE 2020 SR challenge~\cite{lugmayr2020ntire} follows the same setting as the AIM19 dataset, inclusive of an unpaired training set and 100 validation images. Following the methodology of AIM19, we establish the semi-supervised dataset by amalgamating the original unpaired training set with the initial 10 paired images from the validation set. The performance is evaluated on the remaining 90 images within the validation set.

\subsection{Implementation details}
We train all SeNM-VAE models for 300k iterations using the Adam optimizer~\cite{kingma2014adam}. The initial learning rate is set to $10^{-4}$ and halved at 150k, 225k, 270k, and 285k iterations. The batch size is set to 8, consisting of randomly cropped patches of size $64 \times 64$. Batches are sampled randomly from the paired, source, and target domains with equal probability. We apply random flips and rotations to augment the data. The KL regularization parameter $\lambda$ is set to $10^{-7}$. The number of hierarchical layers $L$ is 7. Furthermore, we utilize the KL annealing method~\cite{bowman2015generating} for $\KL(q(\rvz_\rvn|\rvy,\rvz)\|p(\rvz_\rvn|\rvz))$. Specifically, we employ a linear annealing scheme in the first 10k iterations to prevent posterior collapse. To enhance the generation capacity of the VAE model, we incorporate the LPIPS~\cite{zhang2018unreasonable,monakhova2022dancing} loss and GAN loss~\cite{larsen2016autoencoding} to supplement the original L2 loss for noisy image reconstruction. For the SIDD dataset, SeNM-VAE is trained with 0.01\% (10), 0.1\% (96), and 1\% (960) of paired data. In addition, we include all paired data from the SIDD dataset to train our model, using only $\text{Loss}_{p}$ in~\eqref{loss_sep} as our objective function.
% . In this scenario, we only use $\text{Loss}_{p}$ in~\eqref{loss_sep} as our objective function.

\begin{figure}
    \centering
    \includegraphics[width = 1.0\linewidth]{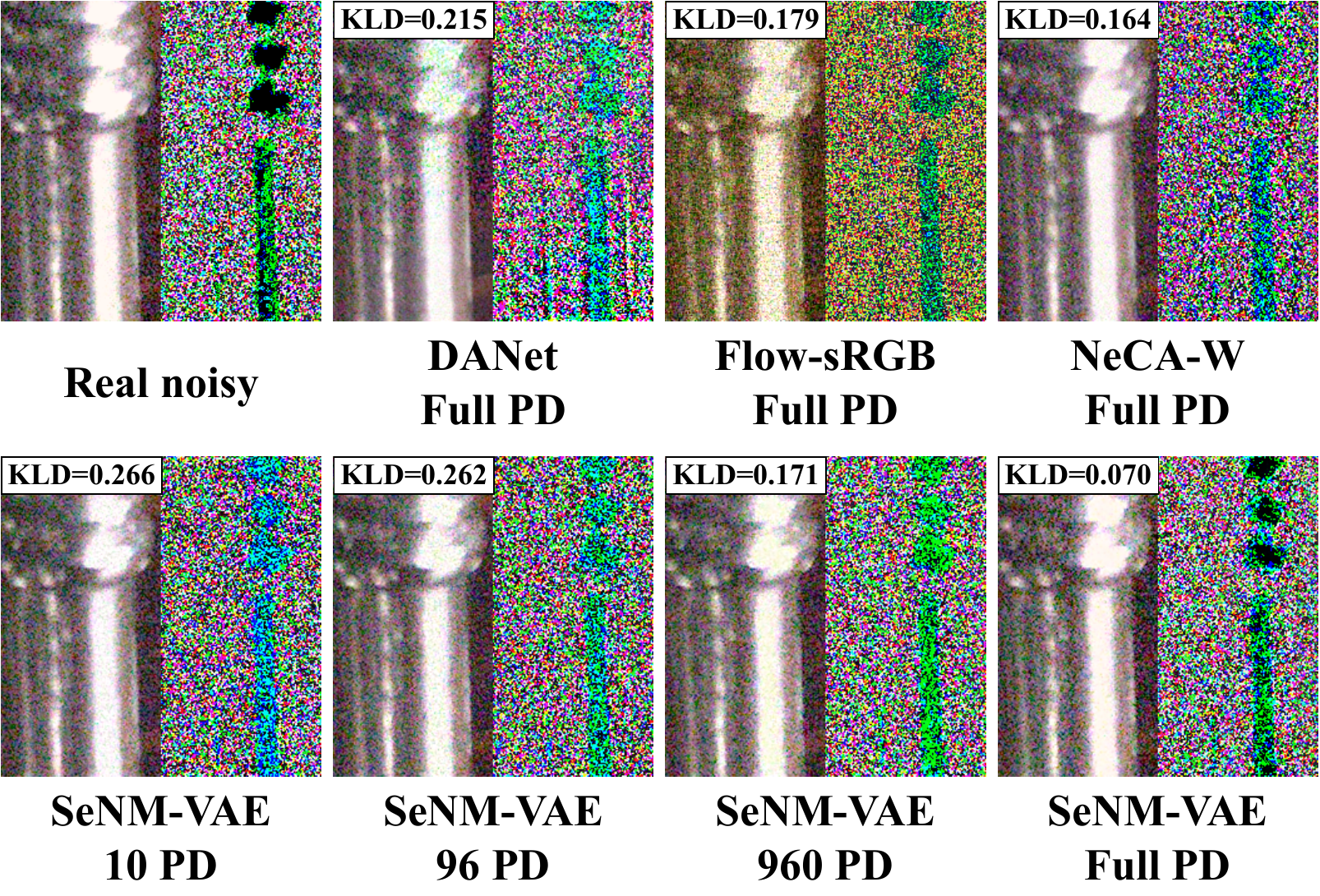}
    \caption{Visual comparison of generated noisy images, "PD" denotes paired data.}
    \label{nm_result}
\end{figure}

\begin{figure}
    \centering
    \includegraphics[width = 1\linewidth]{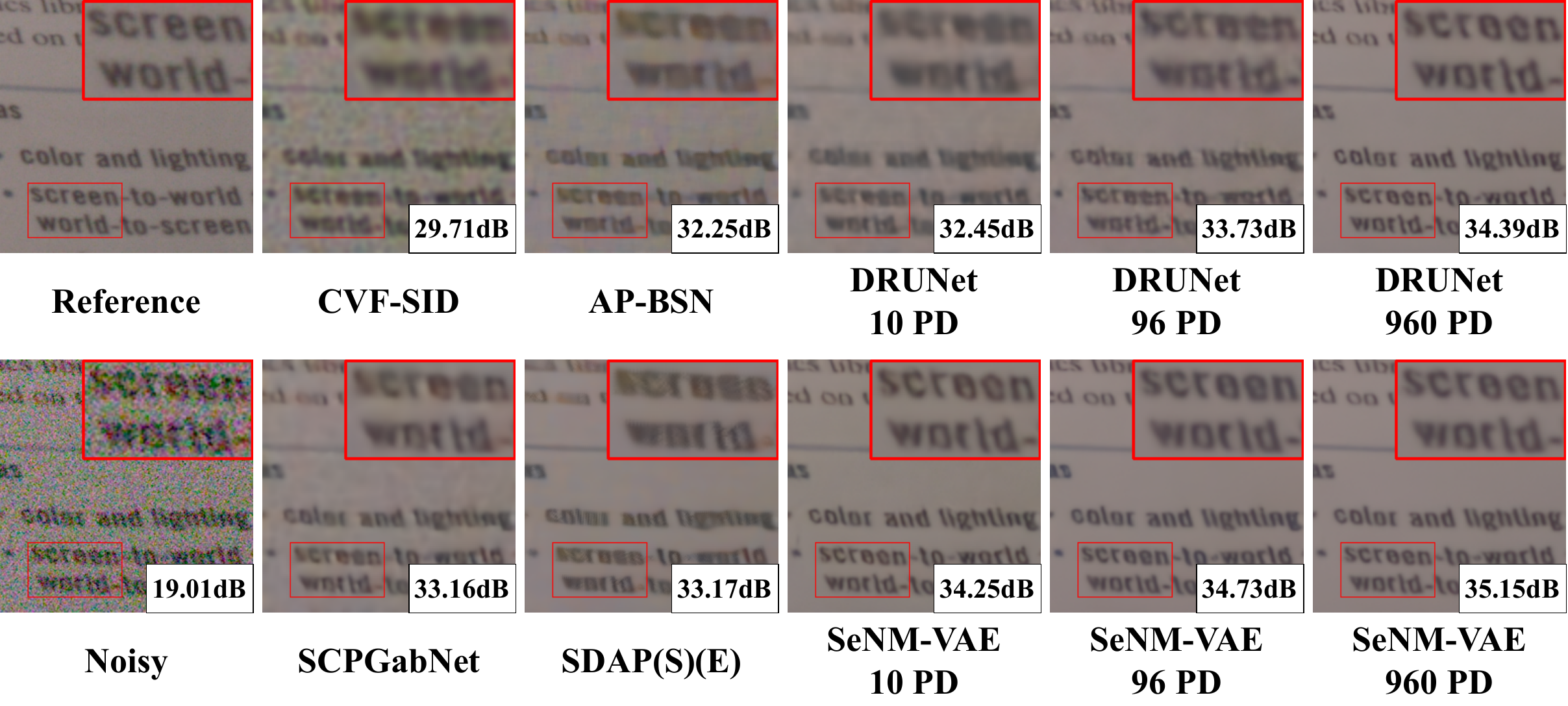}
    \caption{Visual comparison of denoising performance.}
    \label{ds_denoising}
\end{figure}

\begin{figure}
    \centering
    \includegraphics[width = 1\linewidth]{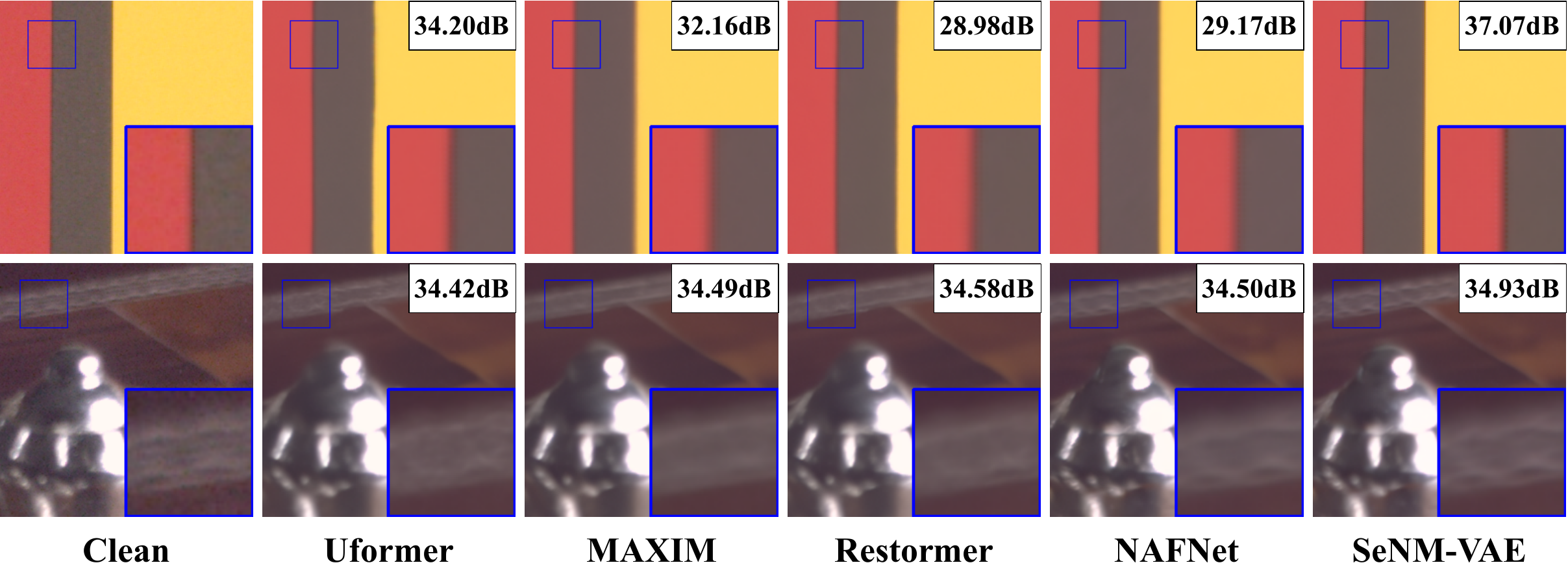}
    \caption{Visual comparison of denoising performance.}
    \label{ft_results}
\end{figure}

\begin{table}
	\setlength{\tabcolsep}{3.15pt}
	\renewcommand{\arraystretch}{0.8}
	\centering	
	\begin{tabular}{ccccc}
		\hline
		\specialrule{0em}{.9pt}{.9pt}
		\rowcolor[gray]{0.95} 
		Method & \# Paired Data & FID $\downarrow$ & KLD $\downarrow$ & PSNR $\uparrow$ \\ 
		\specialrule{0em}{.9pt}{.9pt}
		\hline
		C2N~\cite{jang2021c2n} & \multirow{3}{*}{0} & 33.97 & 0.169 & 34.23 \\
		DeFlow~\cite{wolf2021deflow} & & 39.45 & 0.205 & 33.82 \\
		LUD-VAE~\cite{zheng2022learn} & & 35.31 & 0.108 & 34.91 \\ 
		\hline
		\multirow{3}{*}{\textbf{SeNM-VAE}} & 0.01\% (10) & 17.25 & 0.036 & 36.73 \\ 
        & 0.1\% (96) & 16.76 & 0.027 & 36.92 \\
        & 1\% (960) & \textbf{15.10} & \textbf{0.020} & \textbf{37.28} \\ 
		\hline
        DANet~\cite{yue2020dual} & \multirow{4}{*}{100\%} & 26.22 & 0.081 & 36.25 \\
        Flow-sRGB~\cite{kousha2022modeling} & & 28.60 & 0.047 & 33.24 \\
        NeCA-W~\cite{fu2023srgb} & & 19.96 & 0.030 & 37.04 \\ 
        \textbf{SeNM-VAE} & & \textbf{13.79} & \textbf{0.011} & \textbf{38.29} \\ 
        \hline
        Real noise & 100\% & 0 & 0 & 38.34 \\
        \hline
	\end{tabular}
	\caption{Comparison of noise quality on SIDD validation dataset. DnCNN~\cite{zhang2017beyond} is used as a downstream denoising model.}
    \label{noise_quality}
\end{table}

\begin{table*}
	\setlength{\tabcolsep}{11pt}
	\renewcommand{\arraystretch}{0.8}
	\centering 
	\begin{tabular}{cccccccc}
			\hline
			\specialrule{0em}{.9pt}{.9pt}
			\rowcolor[gray]{0.95} 
			& & \multicolumn{2}{c}{\cellcolor[gray]{0.95}SIDD Validation} &
			\multicolumn{2}{c}{\cellcolor[gray]{0.95}SIDD Benchmark} & \multicolumn{2}{c}{\cellcolor[gray]{0.95}DND Benchmark} \\ \rowcolor[gray]{0.95} 
			\multirow{-2}{*}{\cellcolor[gray]{0.95}Method} & \multirow{-2}{*}{\cellcolor[gray]{0.95}\# Paired Data} & PSNR $\uparrow$ & SSIM $\uparrow$ & PSNR $\uparrow$ & SSIM $\uparrow$ & PSNR $\uparrow$ & SSIM $\uparrow$\\ 
			\specialrule{0em}{.9pt}{.9pt}
			\hline
            N2V~\cite{krull2019noise2void} & \multirow{6}{*}{0} & 29.35 & 0.6510 & 27.68 & 0.668 & - & - \\
            N2S~\cite{batson2019noise2self} & & 30.72 & 0.787 & 29.56 & 0.808 & - & - \\
            CVF-SID~\cite{neshatavar2022cvf} & & 34.17 & 0.872 & 34.71 & 0.917 & 36.50 & 0.924 \\
			AP-BSN + R$^3$~\cite{lee2022ap} & & 35.91 & 0.8815 & 35.97 & 0.925 & 38.09 & 0.937 \\
            SCPGabNet~\cite{lin2023unsupervised} &  & 36.53  & 0.8860 & 36.53 & 0.925 & 38.11 & 0.939 \\
            SDAP(S)(E)~\cite{pan2023random} &  & 37.55  & 0.8943 & 37.53 & 0.936 & 38.56 & 0.940 \\
			\hline
			DRUNet & \multirow{2}{*}{0.01\% (10)} & 34.48 & 0.8658 & 34.45 & 0.909 & 34.37 & 0.904 \\
			\textbf{SeNM-VAE}      & & \textbf{37.96} & \textbf{0.9107} & \textbf{37.93} & \textbf{0.949} & \textbf{38.47} & \textbf{0.946} \\
			\hline
			DRUNet & \multirow{2}{*}{0.1\% (96)} & 37.68 & 0.9053 & 37.63 & 0.944 & 38.16 & 0.942 \\
			\textbf{SeNM-VAE}      & & \textbf{38.91} & \textbf{0.9134} & \textbf{38.85} & \textbf{0.953} & \textbf{39.32} & \textbf{0.951}           \\
			\hline
			DRUNet & \multirow{2}{*}{1\% (960)} & 38.93 & 0.9150 & 38.89 & 0.954 & 38.95 & 0.949 \\
			\textbf{SeNM-VAE}      & & \textbf{39.39} & \textbf{0.9176} & \textbf{39.34} & \textbf{0.956} & \textbf{39.47} & \textbf{0.953}          \\ 
			\hline
			DRUNet & \multirow{3}{*}{100\%}& 39.55 & 0.9187 & 39.51 & 0.957 & 39.52 & 0.952 \\
            VDN~\cite{yue2019variational}  & & 39.29 & 0.9109 & 39.26 & 0.955 & 39.38 & 0.952 \\
            DeamNet~\cite{ren2021adaptive} & & 39.40 & 0.9169 & 39.35 & 0.955 & 39.63 & 0.953 \\
			\hline
		\end{tabular}
	\caption{Comparison of denoising performance on SIDD and DND datasets.}
	\label{sidd_dnd}
\end{table*}

\begin{table}
	\setlength{\tabcolsep}{4.5pt}
	\renewcommand{\arraystretch}{0.8}
	\centering 
	\begin{tabular}{ccccc}
			\hline
			\specialrule{0em}{.9pt}{.9pt}
			\rowcolor[gray]{0.95} 
			& \multicolumn{2}{c}{\cellcolor[gray]{0.95}SIDD Validation} &
			\multicolumn{2}{c}{\cellcolor[gray]{0.95}SIDD Benchmark} \\ \rowcolor[gray]{0.95} 
			\multirow{-2}{*}{\cellcolor[gray]{0.95}Method} & PSNR $\uparrow$ & SSIM $\uparrow$ & PSNR $\uparrow$ & SSIM $\uparrow$ \\ 
			\specialrule{0em}{.9pt}{.9pt}
			\hline
            Uformer~\cite{wang2022uformer} & 39.89 & 0.960 & 39.74 & 0.958 \\
            MAXIM~\cite{tu2022maxim} & 39.96 & 0.960 & 39.84 & 0.959 \\
            Restormer~\cite{zamir2022restormer} & 40.02 & 0.960 & 39.86 & 0.959 \\
            NAFNet~\cite{chen2022simple} & 40.30 & 0.961  & 40.15 & 0.960 \\
            % KBNet~\cite{zhang2023kbnet} & 40.35 & 0.972 &  &  \\
            \hline
            \textbf{SeNM-VAE} & \textbf{40.49} & \textbf{0.962} & \textbf{40.38} & \textbf{0.961} \\
			\hline
		\end{tabular}
	\caption{Comparison of denoising performance on SIDD dataset. SeNM-VAE is trained using the full SIDD dataset and utilized to generate synthetic data for finetuning NAFNet.}
	\label{ft}
\end{table}

\subsection{Noise synthesis}
We first validate the performance of our method through real-world noise modeling tasks on the SIDD dataset.

{\noindent \bf Compared methods.} We compare our SeNM-VAE with three unpaired noise modeling methods, namely C2N~\cite{jang2021c2n}, DeFlow~\cite{wolf2021deflow}, and LUD-VAE~\cite{zheng2022learn}, as well as three fully paired noise modeling methods, namely DANet~\cite{yue2020dual}, Flow-sRGB~\cite{kousha2022modeling}, and NeCA-W~\cite{fu2023srgb}. 

{\noindent \bf Experiment settings and evaluation metrics.} All methods are trained on the SIDD dataset. After training, we apply the trained models to synthesize noisy images using clean images from the SIDD validation set. This allows us to compute the FID~\cite{heusel2017gans} score and the KL divergence between synthetic and real noisy images within the validation set. 
% The FID score measures the distance between real and generated images, with smaller scores indicating better generation quality. The KL divergence between two images can be calculated as follows\begin{equation}
%         \KL(\rmI_1, \rmI_2) = \sum_{i=0}^{255}p(\rmI_1=i) \log \frac{p(\rmI_1=i)}{p(\rmI_2=i)}.
% \end{equation}
Furthermore, using clean images from the SIDD training dataset, we generate noisy images to create synthesized training sets. DnCNN~\cite{zhang2017beyond} models are then trained on these synthesized paired datasets. The performance of DnCNN models is evaluated on the SIDD validation dataset. We employ PSNR to evaluate denoising performance. Higher PSNR values indicate that the noise models are closer to the real noise model, signifying better noise quality.

{\noindent \bf Results.} The results are shown in Table~\ref{noise_quality}. Compared with the unpaired noise modeling approaches, our SeNM-VAE shows remarkable success across all three metrics, even when trained with just 10 paired samples. In contrast to fully paired noise modeling approaches, SeNM-VAE outperforms the SOTA methodology, NeCA-W~\cite{fu2023srgb}, utilizing only 1\% of the original SIDD dataset's paired data. Furthermore, when fully paired data is employed, our method significantly surpasses other noise modeling methods on all three metrics. The results demonstrate the effectiveness of our approach in generating high-quality synthetic noisy images. The visual results are illustrated in Figure~\ref{nm_result}, demonstrating that our method successfully captures the variance change of real-world noise across different regions within the image, particularly when using fully paired data.

\begin{table}
	\setlength{\tabcolsep}{7.3pt}
	\renewcommand{\arraystretch}{0.8}
	\centering	
	\begin{tabular}{cccc}
		\hline
		\specialrule{0em}{.9pt}{.9pt}
		\rowcolor[gray]{0.95} 
		Method      & PSNR $\uparrow$ & SSIM $\uparrow$ & LPIPS $\downarrow$ \\ 
		\specialrule{0em}{.9pt}{.9pt}
		\hline
		FSSR~\cite{fritsche2019frequency}   & 20.97 & 0.5383 & 0.374 \\
		Impressionism~\cite{ji2020real}     & 21.93 & 0.6128 & 0.426 \\
        DASR~\cite{wei2021unsupervised}     & 21.05 & 0.5674 & 0.376 \\
        DeFlow~\cite{wolf2021deflow}        & 21.43 & 0.6003 & 0.349 \\
        LUD-VAE~\cite{zheng2022learn}       & 22.25 & 0.6194 & 0.341 \\
		\hline
        ESRGAN~\cite{wang2018esrgan}        & 21.47 & 0.5748 & 0.353 \\
        \textbf{SeNM-VAE}                   & \textbf{22.48} & \textbf{0.6343} & \textbf{0.333} \\
        \hline
	\end{tabular}
	\caption{Comparison of SR performance on AIM19. ESRGAN and SeNM-VAE are trained with 10 paired data.}
    \label{aim19}
\end{table}

\begin{table}
	\setlength{\tabcolsep}{7.3pt}
	\renewcommand{\arraystretch}{0.8}
	\centering	
	\begin{tabular}{cccc}
		\hline
		\specialrule{0em}{.9pt}{.9pt}
		\rowcolor[gray]{0.95} 
		Method      & PSNR $\uparrow$ & SSIM $\uparrow$ & LPIPS $\downarrow$ \\ 
		\specialrule{0em}{.9pt}{.9pt}
		\hline
		FSSR~\cite{fritsche2019frequency}   & 21.01 & 0.4229 & 0.435 \\
		Impressionism~\cite{ji2020real}     & 25.24 & 0.6740 & 0.230 \\
        DASR~\cite{wei2021unsupervised}     & 22.98 & 0.5093 & 0.379 \\
        DeFlow~\cite{wolf2021deflow}        & 24.95 & 0.6746 & 0.217 \\
        LUD-VAE~\cite{zheng2022learn}       & 25.78 & 0.7196 & 0.220 \\
		\hline
        ESRGAN~\cite{wang2018esrgan}        & 25.05 & 0.6707 & 0.246 \\
        \textbf{SeNM-VAE} & \textbf{25.91} & \textbf{0.7222} & \textbf{0.216} \\
        \hline
	\end{tabular}
	\caption{Comparison of SR performance on NTIRE20. ESRGAN and SeNM-VAE are trained with 10 paired data.}
    \label{ntire20}
\end{table}

\subsection{Downstream denoising}

One significant application of our method is to benefit downstream denoising tasks. After training, we generate synthetic paired data using clean images from the source domain. Then, we employ DRUNet~\cite{zhang2021plug} as the downstream denoising model and train it on both the generated synthetic dataset and the data from the paired domain. 

{\noindent \bf Compared methods.} We compare our semi-supervised denoising method with direct training on the paired domain, and several self-supervised denoising methods, namely N2V~\cite{krull2019noise2void}, N2S~\cite{batson2019noise2self}, CVF-SID~\cite{neshatavar2022cvf}, AP-BSN + R$^3$~\cite{lee2022ap}, SCPGabNet~\cite{lin2023unsupervised}, SDAP(S)(E)~\cite{pan2023random}, and a fully supervised trained DRUNet~\cite{zhang2021plug}, VDN~\cite{yue2019variational}, and DeamNet~\cite{ren2021adaptive}. 

{\noindent \bf Experiment settings and evaluation metrics.} The denoising models are trained for 300k iterations with Adam optimizer. The initial learning rate is $10^{-4}$ and halved every 100k iterations. We evaluate the denoising performance of all the denoising methods on the SIDD validation dataset, the SIDD benchmark dataset, and the DND benchmark dataset, and report PSNR and SSIM~\cite{wang2004image} for each dataset.

{\noindent \bf Results.} The results are presented in Table~\ref{sidd_dnd}. 
The table indicates that our SeNM-VAE approach enhances performance compared to the baseline models on the SIDD and DND datasets.
% The table indicates that our SeNM-VAE approach improves upon the results of the baseline models on the SIDD and DND datasets.
% Compared to self-supervised denoising methods, our method achieves superior results on the SIDD dataset, even when only 10 paired data samples are available. 
Additionally, our method improves upon the results of self-supervised denoising methods on the SIDD dataset, even with access to only 10 paired data samples.
When utilizing 1\% paired data, our method yields results comparable to the fully supervised trained DRUNet model. As such, SeNM-VAE offers an effective strategy to narrow the performance gap between self-supervised and supervised denoising methods. Figure~\ref{ds_denoising} showcases visual results, where our approach achieves sharper edges and more thorough noise removal compared to other methods.

\subsection{Finetune denoising network}

Another application of our method involves generating additional training samples to finetune the denoising network. We utilize our SeNM-VAE, trained with all paired data from the SIDD dataset, to produce extra training data from clean images in the SIDD dataset. We then finetune a pre-trained denoising network, NAFNet~\cite{chen2022simple}. These results are presented in Table~\ref{ft}. The table illustrates that our method can significantly enhance the denoising performance of NAFNet, leading to superior performance on the SIDD dataset. Visual results are shown in Figure~\ref{ft_results}, where our method notably yields sharper edges and preserves more image information.

\subsection{Downstream SR}

We apply our method to simulate the degradation process in real-world SR tasks. Assume the degradation process is $\rvy = \mathcal{D}(\rvx) + \rvn$, where both the downsample operator $\mathcal{D}$ and noise $\rvn$ remain unknown. 
We substitute $\mathcal{D}$ with the Bicubic downsample operator $\mathcal{B}$ and incorporate the disparity between $\mathcal{D}$ and $\mathcal{B}$ into the noise term, yielding $\rvy = \mathcal{B}(\rvx) + \rvn'$, where $\rvn' = \rvn + \mathcal{D}(\rvx) - \mathcal{B}(\rvx)$.
% To overcome the uncertainty regarding $\mathcal{D}$, we substitute it with the Bicubic downsample operator $\mathcal{B}$ and incorporate the disparity between $\mathcal{D}$ and $\mathcal{B}$ into the noise term, yielding $\rvy = \mathcal{B}(\rvx) + \rvn'$, where $\rvn' = \rvn + \mathcal{D}(\rvx) - \mathcal{B}(\rvx)$. 
After training, we generate synthetic low-resolution data and employ ESRGAN~\cite{wang2018esrgan} to train a restoration model.

{\noindent \bf Compared methods.} We compare our semi-supervised SR method with five unpaired degradation modeling methods, namely FSSR~\cite{fritsche2019frequency}, Impressionism~\cite{ji2020real}, DASR~\cite{wei2021unsupervised}, DeFlow~\cite{wolf2021deflow}, LUD-VAE~\cite{zheng2022learn}, and a supervised trained ESRGAN~\cite{wang2018esrgan}. 
% All aforementioned unpaired degradation modeling methods learn the degradation process from non-matching data and generate synthetic data pairs, which are used to train a ESRGAN~\cite{wang2018esrgan} based SR model in the supervised manner. 

{\noindent \bf Experiment settings and evaluation metrics.} All the SR models are trained for 60k iterations with Adam optimizer. We evaluate each SR method on the final 90 images from the AIM19 and NTIRE20 validation datasets. Performance metrics, including PSNR, SSIM, and LPIPS~\cite{zhang2018unreasonable}, are reported for both datasets.

{\noindent \bf Results.} The results are detailed in Table~\ref{aim19} and Table~\ref{ntire20}. Our method surpasses both the supervised ESRGAN model and the unpaired degradation modeling methods, which highlights the effectiveness of our model in leveraging a limited amount of paired data alongside unpaired data to enhance the generation of high-quality training samples. 

% These tables highlight the enhancement achieved by our semi-supervised method in comparison to the baseline SR model. Our method also surpasses other unpaired degradation modeling methods, suggesting that SeNM-VAE can generate more authentic training pairs with the utilization of 10 additional paired data, thereby mitigating the disparity between self-supervised and supervised learning. 

\subsection{Ablation study and discussions}

\begin{table}
	\setlength{\tabcolsep}{13pt}
	\renewcommand{\arraystretch}{0.8}
	\centering	
	\begin{tabular}{cccc}
		\hline
		\specialrule{0em}{.9pt}{.9pt}
		\rowcolor[gray]{0.95} 
		Total Loss      & FID $\downarrow$ & KLD $\downarrow$ & PSNR $\uparrow$ \\ 
		\specialrule{0em}{.9pt}{.9pt}
		\hline
		w/o recon. $\rvx$      & 24.26 & 0.070 & 35.60 \\
		w/ recon. $\rvx$       & 17.25 & 0.036 & 36.73 \\
		\hline
	\end{tabular}
	\caption{Ablation on the reconstruction loss for the clean image $\rvx$ on SIDD validation dataset with 10 paired data.}
    \label{clean_recon}
\end{table}

{\noindent \bf Ablation on reconstruction loss for clean image $\rvx$.} We perform an ablation experiment to the additional loss $-\E_{q(\rvz|\rvx,\rvy)} \log p(\rvx|\rvz)$ in~\eqref{loss} on the SIDD dataset with 10 paired samples, and the results are shown in Table~\ref{clean_recon}. 
% The table shows that including this reconstruction loss largely improves the noise modeling performance. 
The table demonstrates an improvement in noise modeling performance with the inclusion of this reconstruction loss.
% One reason is that the inclusion of this reconstruction loss serves as the loss function $\text{Loss}_{s}$ for source domain data in~\eqref{loss_sep}, allowing effective utilization of source domain data.
One reason is that the reconstruction loss acts as $\text{Loss}_{s}$ for source domain data in~\eqref{loss_sep}, allowing effective utilization of source domain data.
Moreover, the reconstruction loss enables the transformation from noisy image $\rvy$ to the clean image $\rvx$, aiding in the disentanglement of image information $\rvz$ from noisy information $\rvz_\rvn$, which further enhances the model's ability to learn the noise distribution.

{\noindent \bf Analysis on the training domains.} To demonstrate the efficacy of our model in utilizing information from unpaired data domains, we conduct experiments using different numbers of unpaired data on the SIDD dataset with 10 paired data. Specifically, we randomly select 0, 96, and 960 images for the source and target domains to illustrate our model's capacity to exploit information from unpaired data domains. Then, we conduct ablation studies using only two domains to understand the roles of the source and target domains in the training process. The results are summarized in Table~\ref{domain_ablation}. From the table, we find that the noise modeling performance improves with increasing unpaired data, and incorporating all three domains yields the best noise modeling results. The findings indicate our model's effectiveness in leveraging information from unpaired datasets. 

\begin{table}
	\setlength{\tabcolsep}{6.2pt}
	\renewcommand{\arraystretch}{0.8}
	\centering	
	\begin{tabular}{cccccc}
		\hline
		\specialrule{0em}{.9pt}{.9pt}
		\rowcolor[gray]{0.95} 
		\multicolumn{3}{c}{\# Data} & & & \\
        \rowcolor[gray]{0.95}  Paired & Source & Target & \multirow{-2}{*}{\cellcolor[gray]{0.95}FID $\downarrow$} & 
        \multirow{-2}{*}{\cellcolor[gray]{0.95}KLD $\downarrow$} & \multirow{-2}{*}{\cellcolor[gray]{0.95}PSNR $\uparrow$}  \\ 
		\specialrule{0em}{.9pt}{.9pt}
		\hline
		10 & 0   & 0        & 19.79 & 0.073 & 35.97 \\
        10 & 96  & 96       & 17.52 & 0.054 & 36.67 \\
		10 & 960 & 960      & 17.54 & 0.036 & 36.78 \\
        \hline
        10 & 100\% & 0      & 22.98 & 0.043 & 36.23 \\
        10 & 0 & 100\%      & 19.11 & 0.050 & 36.24 \\
		10 & 100\% & 100\%  & 17.25 & 0.036 & 36.73 \\
		\hline
	\end{tabular}
	\caption{Comparison of noise quality on SIDD validation dataset with different numbers of unpaired samples.}
	\label{domain_ablation}
\end{table}

{\noindent \bf Analysis on KL weight $\lambda$.} We perform a parameter analysis for the KL weight coefficient $\lambda$. The results are presented in Figure~\ref{ablation_klw}. The figure shows that setting $\lambda$ in the range of $10^{-6}$ to $10^{-7}$ leads to better noise modeling results than other cases. 

\begin{figure}
	\centering
        \includegraphics[width=\linewidth]{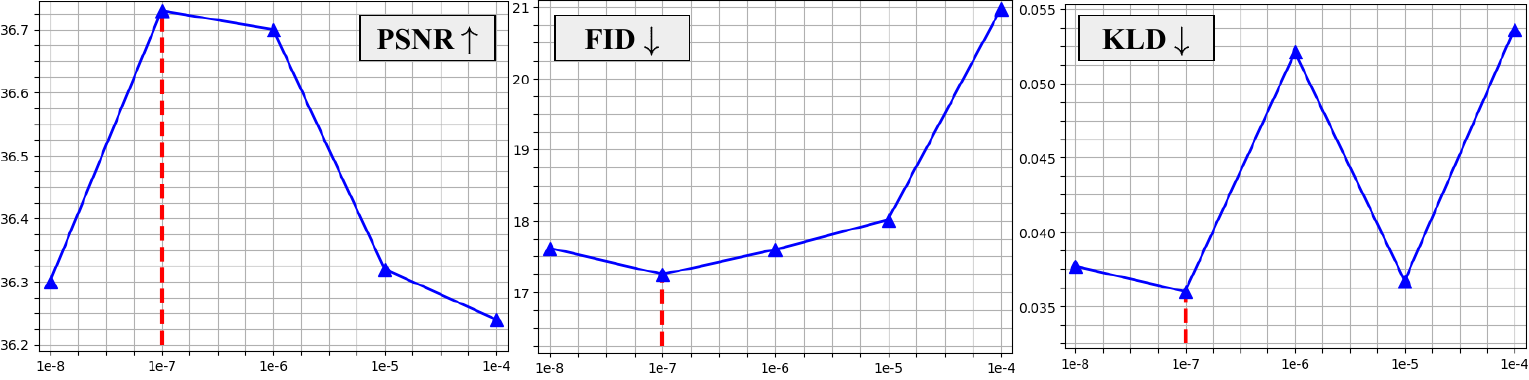}
	\caption{Parameter analysis of KL weight $\lambda$ on SIDD validation dataset with 10 paired data.}
	\label{ablation_klw}
\end{figure}

\section{Conclusion}

This paper presents SeNM-VAE, a semi-supervised noise modeling approach based on deep variational inference. The proposed method employs a latent variable model to capture the conditional distribution between corrupted and clean images, allowing for the transformation from a clean image to its corrupted counterpart. Our approach decomposes the objective function, enabling training with both paired and unpaired datasets. We apply SeNM-VAE to real-world noise modeling and downstream denoising and super-resolution tasks. Our method further improves upon other degradation modeling methods and achieves the best performance on the SIDD dataset.

% The derivation of SeNM-VAE is based on the assumption of a mixture inference model within the VAE framework, which provides the mathematical rationale for our approach. We apply SeNM-VAE to real-world noise modeling and downstream denoising and super-resolution tasks. Our method outperforms other degradation modeling methods and achieves new SOTA denoising results on the SIDD dataset.

\section*{Acknowledgements}

This work was supported by the National Key R\&D Program of China (No. 2021YFA1001300), the National Natural Science Foundation of China (No. 12271291).

{
    \small
    \bibliographystyle{ieeenat_fullname}
    \bibliography{main}
}

% WARNING: do not forget to delete the supplementary pages from your submission 
\clearpage
\setcounter{page}{1}
\maketitlesupplementary

% \section{Proof for Proposition 1}

% \begin{prop}\label{kl_bound}
% Let $q(\rvz|\rvx,\rvy)$ be a mixture model of $q(\rvz|\rvx)$ and $q(\rvz|\rvy)$:
% \begin{equation}
% 	q(\rvz|\rvx,\rvy) = p_1 q(\rvz | \rvx) + p_2 q(\rvz | \rvy),
% \end{equation}
% then:
% \begin{equation}\label{kl_ineq}
% \begin{aligned}
%     \KL(q(\rvz|\rvx,\rvy) \| p(\rvz|\rvx)) \leq & p_1 \KL (q(\rvz|\rvx) \| p(\rvz|\rvx)) \\
%     +& p_2 \KL (q(\rvz|\rvy) \| p(\rvz|\rvx)).
% \end{aligned}
% \end{equation}
% Moreover, suppose that $q(\rvz|\rvx)=p(\rvz|\rvx)$ by sharing the same neural network. Then:
% \begin{equation}
% \KL(q(\rvz|\rvx,\rvy) \| p(\rvz|\rvx)) \leq p_2 \KL (q(\rvz|\rvy) \| q(\rvz|\rvx))
% \end{equation}
% \end{prop}

% \section{Architecture of the degradation level prediction network}

% \section{Visual results}
\appendix

\section{Detailed derivation}
The derivation of Equation~\eqref{cELBO} in the main paper is elucidated in detail herein. By introducing an inference model $q\left(\rvz, \rvz_\rvn | \rvx, \rvy\right)$, we decompose $\log p\left(\rvy | \rvx\right)$ into the following two terms:
\begin{equation}
    \begin{aligned}
        \log p&\left(\rvy | \rvx\right) = \mathbb E_{q\left(\rvz, \rvz_\rvn | \rvx, \rvy\right)}\log\frac{p\left(\rvy, \rvz, \rvz_\rvn | \rvx\right)}{q\left(\rvz, \rvz_\rvn | \rvx, \rvy\right)} \\
        +& \left(\log p\left(\rvy | \rvx\right) - \mathbb E_{q\left(\rvz, \rvz_\rvn | \rvx, \rvy\right)}\log\frac{p\left(\rvy, \rvz, \rvz_\rvn | \rvx\right)}{q\left(\rvz, \rvz_\rvn | \rvx, \rvy\right)}\right),
    \end{aligned}
\end{equation}
where the first term represents the cELBO. The second term can be expressed as follows:
\begin{equation}
    \begin{aligned}
        & \log p\left(\rvy | \rvx\right) - \mathbb E_{q\left(\rvz, \rvz_\rvn | \rvx, \rvy\right)}\log\frac{p\left(\rvy, \rvz, \rvz_\rvn | \rvx\right)}{q\left(\rvz, \rvz_\rvn | \rvx, \rvy\right)} \\
        =& \mathbb E_{q\left(\rvz, \rvz_\rvn | \rvx, \rvy\right)} \left[\log p\left(\rvy | \rvx\right) - \log\frac{p\left(\rvy | \rvx\right)p\left(\rvz, \rvz_\rvn | \rvx, \rvy\right)}{q\left(\rvz, \rvz_\rvn | \rvx, \rvy\right)}\right] \\
        =& \mathbb E_{q\left(\rvz, \rvz_\rvn | \rvx, \rvy\right)} \log \frac{q\left(\rvz, \rvz_\rvn | \rvx, \rvy\right)}{p\left(\rvz, \rvz_\rvn | \rvx, \rvy\right)} \\
        =& \KL \left(q\left(\rvz, \rvz_\rvn | \rvx, \rvy\right) || p\left(\rvz, \rvz_\rvn | \rvx, \rvy\right)\right).
    \end{aligned}
\end{equation}
According to the proposed graphical model (as depicted in Figure 1a in the main paper), we have
\begin{equation}
	\begin{aligned}
            p\left(\rvy, \rvz, \rvz_\rvn | \rvx\right) =& p\left(\rvz | \rvx\right)p\left(\rvz_\rvn | \rvx, \rvz\right)p\left(\rvy | \rvx, \rvz, \rvz_\rvn\right) \\
            =& p\left(\rvz | \rvx\right)p\left(\rvz_\rvn | \rvz\right)p\left(\rvy | \rvz, \rvz_\rvn\right), \\
            p\left(\rvz, \rvz_\rvn | \rvx, \rvy\right) =& p\left(\rvz | \rvx, \rvy\right)p\left(\rvz_\rvn | \rvx, \rvy, \rvz\right) \\
            =& p\left(\rvz | \rvx, \rvy\right)p\left(\rvz_\rvn | \rvy, \rvz\right).
		% & p(\rvy, \rvz, \rvz_\rvn|\rvx) = p(\rvz|\rvx) p(\rvz_\rvn|\rvz) p(\rvy | \rvz,\rvz_\rvn), \\
		% & q(\rvz, \rvz_\rvn | \rvx, \rvy) = q(\rvz | \rvx, \rvy) q(\rvz_\rvn | \rvy, \rvz).
	\end{aligned}
\end{equation}
To maintain consistency with the decomposition of $p\left(\rvz, \rvz_\rvn | \rvx, \rvy\right)$, we choose
\begin{equation}
    q(\rvz, \rvz_\rvn | \rvx, \rvy) = q(\rvz | \rvx, \rvy) q(\rvz_\rvn | \rvy, \rvz).
\end{equation}
Consequently, the cELBO can be further factorized as
\begin{equation}
    \begin{aligned}
        & \mathbb E_{q\left(\rvz, \rvz_\rvn | \rvx, \rvy\right)}\log\frac{p\left(\rvy, \rvz, \rvz_\rvn | \rvx\right)}{q\left(\rvz, \rvz_\rvn | \rvx, \rvy\right)} \\
        =& \mathbb E_{q\left(\rvz, | \rvx, \rvy\right)q\left(\rvz_\rvn | \rvy, \rvz\right)} \log\frac{p\left(\rvz | \rvx\right) p\left(\rvz_\rvn | \rvz\right) p\left(\rvy | \rvz, \rvz_\rvn\right)}{q\left(\rvz | \rvx, \rvy\right) q\left(\rvz_\rvn | \rvy, \rvz\right)} \\
        =& \mathbb E_{q\left(\rvz, | \rvx, \rvy\right)q\left(\rvz_\rvn | \rvy, \rvz\right)} \log p\left(\rvy | \rvz_\rvn\right) + \mathbb E_{q\left(\rvz, | \rvx, \rvy\right)} \log \frac{p\left(\rvz | \rvx\right)}{q\left(\rvz | \rvx, \rvy\right)} \\
        &+ \mathbb E_{q\left(\rvz, | \rvx, \rvy\right)q\left(\rvz_\rvn | \rvy, \rvz\right)} \log \frac{p\left(\rvz_\rvn | \rvz\right)}{q\left(\rvz_\rvn | \rvy, \rvz\right)} \\
        =& \mathbb E_{q\left(\rvz, | \rvx, \rvy\right)q\left(\rvz_\rvn | \rvy, \rvz\right)} \log p\left(\rvy | \rvz_\rvn\right) - \KL\left(q\left(\rvz | \rvx, \rvy\right) || p\left(\rvz | \rvx\right)\right) \\
        &- \mathbb E_{q\left(\rvz | \rvx, \rvy\right)}\KL\left(q\left(\rvz_\rvn | \rvy, \rvz\right) || p\left(\rvz_\rvn | \rvy, \rvz\right)\right).
    \end{aligned}
\end{equation}

\section{Proof for Proposition~\ref{kl_bound}}
\label{supp_sec:proof_for_prop1}

\begin{prop*}
Let $q(\rvz|\rvx,\rvy)$ be a mixture model of $q(\rvz|\rvx)$ and $q(\rvz|\rvy)$:
\begin{equation}
	q(\rvz|\rvx,\rvy) = p_1 q(\rvz | \rvx) + p_2 q(\rvz | \rvy),
\end{equation}
then:
\begin{equation}
% \label{kl_ineq}
\begin{aligned}
    \KL(q(\rvz|\rvx,\rvy) \| p(\rvz|\rvx)) \leq & p_1 \KL (q(\rvz|\rvx) \| p(\rvz|\rvx)) \\
    +& p_2 \KL (q(\rvz|\rvy) \| p(\rvz|\rvx)).
\end{aligned}
\end{equation}
Moreover, suppose that $q(\rvz|\rvx)=p(\rvz|\rvx)$ by sharing the same neural network. Then:
\begin{equation}
\label{kl_ineq_re}
\KL(q(\rvz|\rvx,\rvy) \| p(\rvz|\rvx)) \leq p_2 \KL (q(\rvz|\rvy) \| q(\rvz|\rvx))
\end{equation}
\end{prop*}

\begin{proof}
Using the log-sum inequality, we have:
\begin{equation}
    \begin{aligned}
           & \KL(q(\rvz|\rvx,\rvy) \| p(\rvz|\rvx)) \\
         = & \KL(p_1 q(\rvz | \rvx) + p_2 q(\rvz | \rvy) \| p_1 p(\rvz|\rvx) + p_2 p(\rvz|\rvx) ) \\
         = & \int \left( p_1 q(\rvz | \rvx) + p_2 q(\rvz | \rvy) \right) \log \frac{p_1 q(\rvz | \rvx) + p_2 q(\rvz | \rvy)}{p_1 p(\rvz|\rvx) + p_2 p(\rvz|\rvx)} d\rvz \\
         \leq & \int p_1 q(\rvz|\rvx) \log \frac{p_1 q(\rvz|\rvx)}{p_1 p(\rvz|\rvx)} + p_2 q(\rvz|\rvy) \log \frac{p_2 q(\rvz|\rvy)}{p_2 p(\rvz|\rvx)} d\rvz \\
         = & p_1 \KL (q(\rvz|\rvx) \| p(\rvz|\rvx)) + p_2 \KL (q(\rvz|\rvy) \| p(\rvz|\rvx)),
    \end{aligned}
\end{equation}
then~\eqref{kl_ineq} holds. Furthermore, since we can parameterize $q(\rvz|\rvx)$ and $p(\rvz|\rvx)$ with the same distribution, then $q(\rvz|\rvx) = p(\rvz|\rvx)$, and~\eqref{kl_ineq} is reduced to~\eqref{kl_ineq_re}.
\end{proof}

\begin{figure}
	\centering
	\includegraphics[width=1\linewidth]{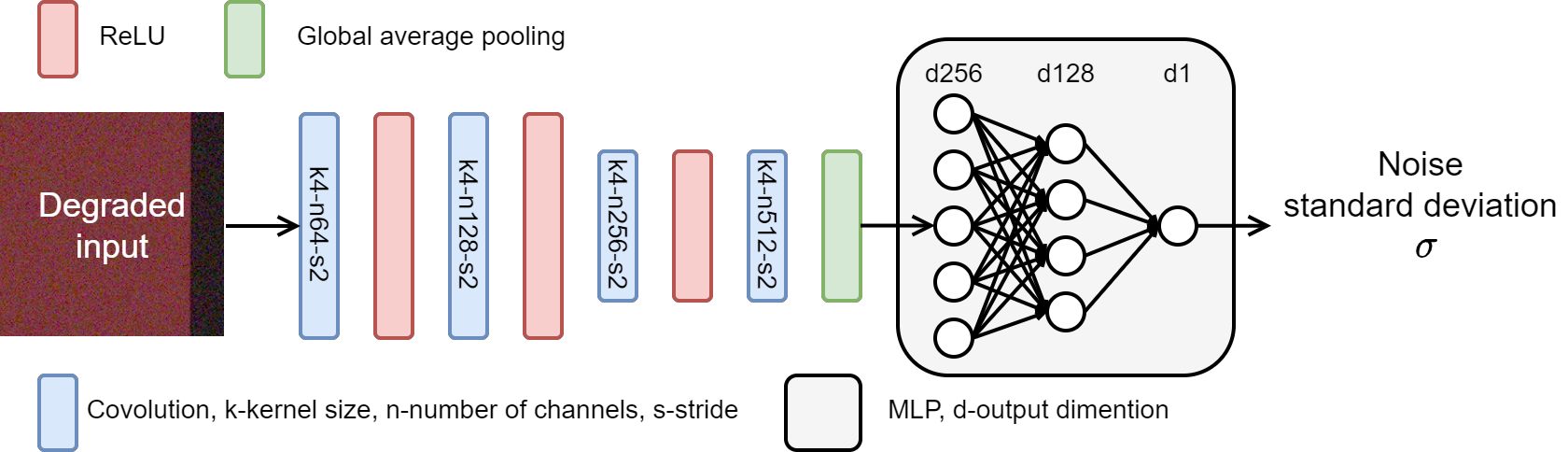}
	\caption{Architecture of degradation level prediction network.}
	\label{nln_arch}
\end{figure}

\section{Architecture of the degradation level prediction network}
We incorporate the standard deviation of the noise, along with the noisy image from the target domain, into our SeNM-VAE model to enable controlled generation of degradation levels. Specifically, we concatenate the degradation level with $\rvb_{\rvn}^l$ (see Equation~\eqref{dec_fea_zn} in the main paper) to enable conditional generation during both training and generation processes. Since the noisy image from the target domain lacks the corresponding clean image, its degradation level cannot be directly determined. Therefore, we introduce a degradation level prediction network trained on data from the paired domain and use it to predict the noise standard deviation for data from the target domain. The architecture of this network is illustrated in Figure~\ref{nln_arch}. Our approach has been shown to successfully generate images with varying input noise levels, as demonstrated in Figure~\ref{change_nl}.

\begin{figure*}
	\centering
	\begin{tabular}{c@{\hspace{0.005\linewidth}}c@{\hspace{0.005\linewidth}}c@{\hspace{0.005\linewidth}}c@{\hspace{0.005\linewidth}}c@{\hspace{0.005\linewidth}}c}
		
		\includegraphics[width=0.16\linewidth]{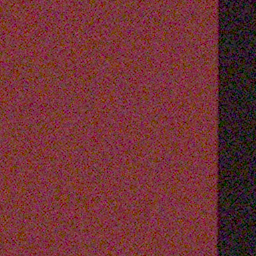} & 
		\includegraphics[width=0.16\linewidth]{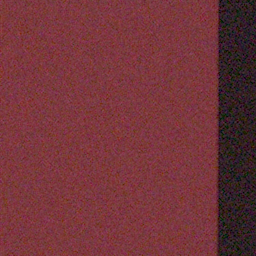} & 
		\includegraphics[width=0.16\linewidth]{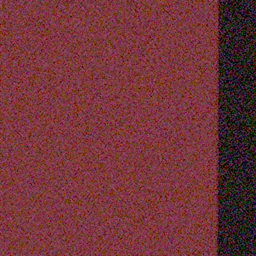} & 
		\includegraphics[width=0.16\linewidth]{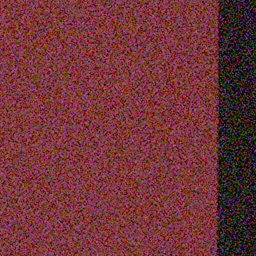} & 
		\includegraphics[width=0.16\linewidth]{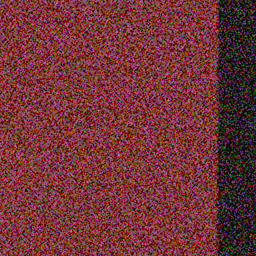} &
		\includegraphics[width=0.16\linewidth]{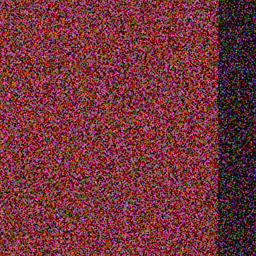} \\
		
		\includegraphics[width=0.16\linewidth]{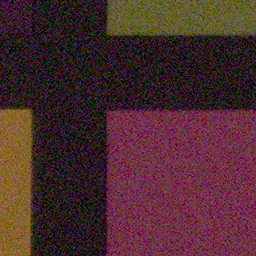} & 
		\includegraphics[width=0.16\linewidth]{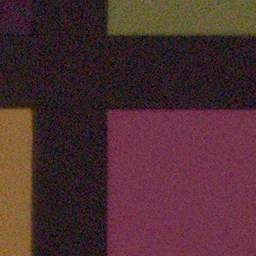} & 
		\includegraphics[width=0.16\linewidth]{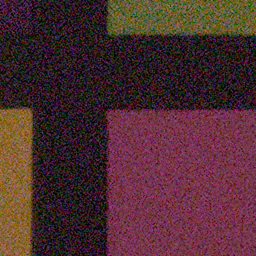} & 
		\includegraphics[width=0.16\linewidth]{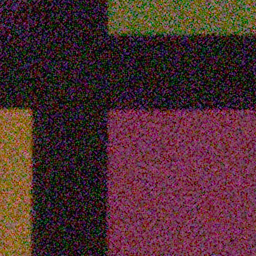} & 
		\includegraphics[width=0.16\linewidth]{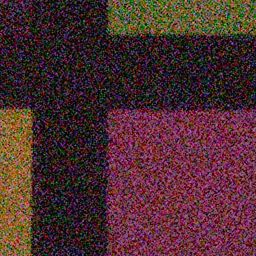} &
		\includegraphics[width=0.16\linewidth]{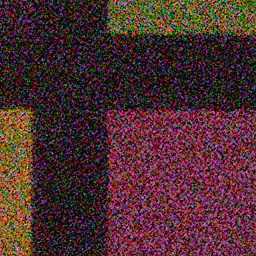} \\
		
		\small{Real noisy} & \small{Gen noisy, $\sigma=10$} & \small{Gen noisy, $\sigma=20$} & \small{Gen noisy, $\sigma=30$} & \small{Gen noisy, $\sigma=40$} & \small{Gen noisy, $\sigma=50$} \\
	\end{tabular}
	\caption{Visual results of degradation level controllable generation on SIDD validation dataset, $\sigma$ denotes the input degradation level. The model is trained with 10 paired data on the SIDD dataset.}
	\label{change_nl}
\end{figure*}

\section{Experiment}

\subsection{Implementation details}

{\noindent \bf Computation of KL divergence.} We use KL divergence to evaluate the fidelity of generated noisy images. The KL divergence between two images can be calculated as follows:
\begin{equation}
    \KL(\rmI_1, \rmI_2) = \sum_{i=0}^{255}p(\rmI_1=i) \log \frac{p(\rmI_1=i)}{p(\rmI_2=i)}.
\end{equation}

{\noindent \bf Training details of DnCNN.} We train all DnCNN~\cite{zhang2017beyond} models for 300k iterations using the Adam optimizer~\cite{kingma2014adam}. The initial learning rate is set to $10^{-4}$ and halved every 100k iterations. The batch size is 64, consisting of randomly cropped patches of size $40 \times 40$. Random flips and rotations are applied to augment the data. We evaluate the performance every 5k iterations on the SIDD validation dataset and select the model with the highest PSNR to evaluate on the benchmark set.

{\noindent \bf Training details of DRUNet.} All DRUNet~\cite{zhang2021plug} models are trained for 300k iterations using the Adam optimizer~\cite{kingma2014adam}. The initial learning rate is set to $10^{-4}$ and halved every 100k iterations. The batch size is 16, consisting of randomly cropped patches of size $128 \times 128$.  We augment the data by applying random flips and rotations. We evaluate the performance every 5k iterations on the SIDD validation dataset and select the model with the highest PSNR to evaluate on the benchmark set.

{\noindent \bf Training details of NAFNet.} We finetune the pre-trained NAFNet~\cite{chen2022simple} on synthesized training set. The model is trained for 400k iterations with Adam optimizer~\cite{kingma2014adam}. The initial learning rate is set to $10^{-5}$, and we use the cosine learning rate decay schedule. The batch size is 2, and the patch size is $256 \times 256$. We evaluate the denoising performance every 20k iterations on the SIDD validation dataset and select the model with the highest PSNR to evaluate on the benchmark set.

{\noindent \bf Training details of ESRGAN.} We use the training code from Impressionism~\cite{ji2020real} and train the ESRGAN~\cite{wang2018esrgan} model for 60k iterations. The initial learning rate is set to $10^{-4}$ and halved at 5k, 10k, 20k, 30k iterations. The batch size is 16, consisting of randomly cropped patches of size $128 \times 128$. Random flips and rotations are applied to augment the data. We use the model at 60k iterations to evaluate the final performance.

\subsection{Benchmark results}

We replenish Table~\ref{noise_quality} in the main paper with the denoising results of DnCNN~\cite{zhang2017beyond} on the SIDD and DND benchmarks. These results are shown in Table~\ref{noise_quality_sidd_bench} and Table~\ref{noise_quality_dnd_bench}. Compared to the unpaired noise modeling methods, our method yields superior denoising results, even with 10 paired samples. Notably, as the number of paired samples increases, our method consistently exhibits the most effective denoising performance across both benchmarks. This further attests to the competitive advantage of our method in producing high-quality synthesized noisy images.

\begin{table}
	\setlength{\tabcolsep}{7pt}
	\renewcommand{\arraystretch}{0.8}
	\centering	
	\begin{tabular}{cccc}
		\hline
		\specialrule{0em}{.9pt}{.9pt}
		\rowcolor[gray]{0.95} 
		Method & \# Paired Data & PSNR $\uparrow$ & SSIM $\uparrow$ \\ 
		\specialrule{0em}{.9pt}{.9pt}
		\hline
		C2N~\cite{jang2021c2n} & \multirow{3}{*}{0} & 33.95 & 0.878 \\
		DeFlow~\cite{wolf2021deflow}      &         & 33.81 & 0.897 \\
		LUD-VAE~\cite{zheng2022learn}     &         & 34.82 & 0.926 \\ 
		\hline
		\multirow{3}{*}{\textbf{SeNM-VAE}} & 0.01\% (10) & 36.68 & 0.931 \\ 
                                           & 0.1\% (96)  & 36.89 & 0.928 \\
                                           & 1\% (960)   & \textbf{37.24} & \textbf{0.938} \\ 
		\hline
        DANet~\cite{yue2020dual} & \multirow{4}{*}{100\%} & 36.20 & 0.925 \\
        Flow-sRGB~\cite{kousha2022modeling} &             & 33.24 & 0.876 \\
        NeCA-W~\cite{fu2023srgb}            &             & 36.95 & 0.935 \\ 
        \textbf{SeNM-VAE}                   &             & \textbf{38.27} & \textbf{0.946} \\ 
        \hline
        Real noise                          & 100\%       & 38.31 & 0.946 \\
        \hline
	\end{tabular}
	\caption{Comparison of denoising results on SIDD benchmark. DnCNN~\cite{zhang2017beyond} is used as a downstream denoising model.}
    \label{noise_quality_sidd_bench}
\end{table}

\begin{table}
	\setlength{\tabcolsep}{7pt}
	\renewcommand{\arraystretch}{0.8}
	\centering	
	\begin{tabular}{cccc}
		\hline
		\specialrule{0em}{.9pt}{.9pt}
		\rowcolor[gray]{0.95} 
		Method & \# Paired Data & PSNR $\uparrow$ & SSIM $\uparrow$ \\ 
		\specialrule{0em}{.9pt}{.9pt}
		\hline
		C2N~\cite{jang2021c2n} & \multirow{3}{*}{0} & 36.08 & 0.903 \\
		DeFlow~\cite{wolf2021deflow}      &         & 36.71 & 0.923 \\
		LUD-VAE~\cite{zheng2022learn}     &         & 37.60 & 0.933 \\ 
		\hline
		\multirow{3}{*}{\textbf{SeNM-VAE}} & 0.01\% (10) & 37.94 & 0.936 \\ 
                                           & 0.1\% (96)  & 38.21 & 0.942 \\
                                           & 1\% (960)   & \textbf{38.44} & \textbf{0.943} \\ 
		\hline
        DANet~\cite{yue2020dual} & \multirow{4}{*}{100\%} & 38.21 & 0.943 \\
        Flow-sRGB~\cite{kousha2022modeling} &             & 36.09 & 0.895 \\
        NeCA-W~\cite{fu2023srgb}            &             & 38.70 & 0.946 \\ 
        \textbf{SeNM-VAE}                   &             & \textbf{39.09} & \textbf{0.950} \\ 
        \hline
        Real noise                          & 100\%       & 38.83 & 0.949 \\
        \hline
	\end{tabular}
	\caption{Comparison of denoising results on DND benchmark. DnCNN~\cite{zhang2017beyond} is used as a downstream denoising model.}
    \label{noise_quality_dnd_bench}
\end{table}

\subsection{Model complexity}

The proposed SeNM-VAE can effectively utilize a limited amount of paired data together with unpaired data to enhance the generation of high-quality training samples, without necessitating extensive computational resources. Specifically, the total number of parameters in our model amounts to $9.946$M, with a total FLOPs of $617.36$G required to generate a single $256 \times 256 \times 3$ image. Additionally, training can be completed within approximately 2 days on a single Nvidia 2080 Ti GPU on the SIDD dataset. During the inference stage, generating 1,280 images takes around 31 seconds.

\subsection{Training stability}

The overall training objective of SeNM-VAE consists of three parts. Firstly, it involves maximizing the conditional log-likelihood function, $\log p\left(\rvy | \rvx\right)$, through variational inference methods and the proposed mixture model, encompassing three key elements:
\begin{equation}
    \begin{aligned}
        & \mathbb E_{q\left(\rvz | \rvx, \rvy\right)}\KL\left(q\left(\rvz_\rvn | \rvy, \rvz\right) || p\left(\rvz_\rvn | \rvz\right)\right) \\
        -& \mathbb E_{q\left(\rvz | \rvx, \rvy\right) q\left(\rvz_\rvn | \rvy, \rvz\right)} \log p\left(\rvy | \rvz, \rvz_\rvn\right) \\
        +& \lambda \KL\left(q\left(\rvz | \rvy\right) || q\left(\rvz | \rvx\right)\right).
    \end{aligned}
    \label{eq:loss_log_likelihood}
\end{equation}
Another component comprises a regularization term, namely $\mathbb E_{q\left(\rvz | \rvx, \rvy\right)} \log p\left(\rvx | \rvz\right)$. This term plays a crucial role in enhancing the reconstruction capability of the source domain data, especially since the terms in \eqref{eq:loss_log_likelihood} do not directly regulate the source domain data. To augment the generative capacity of the VAE model, we incorporate the LPIPS loss and GAN loss to complement the loss function for noisy image reconstruction, constituting the third part of the loss function. In our experiments, we train our model using the conventional ADAM optimizer~\cite{kingma2014adam} with its default settings. Employing standard training techniques in VDVAE~\cite{child2020very}, we observe stable convergence performance, as depicted in Figure~\ref{fig:loss_curve}.

\begin{figure}
    \centering
    \includegraphics[width = 0.8\linewidth]{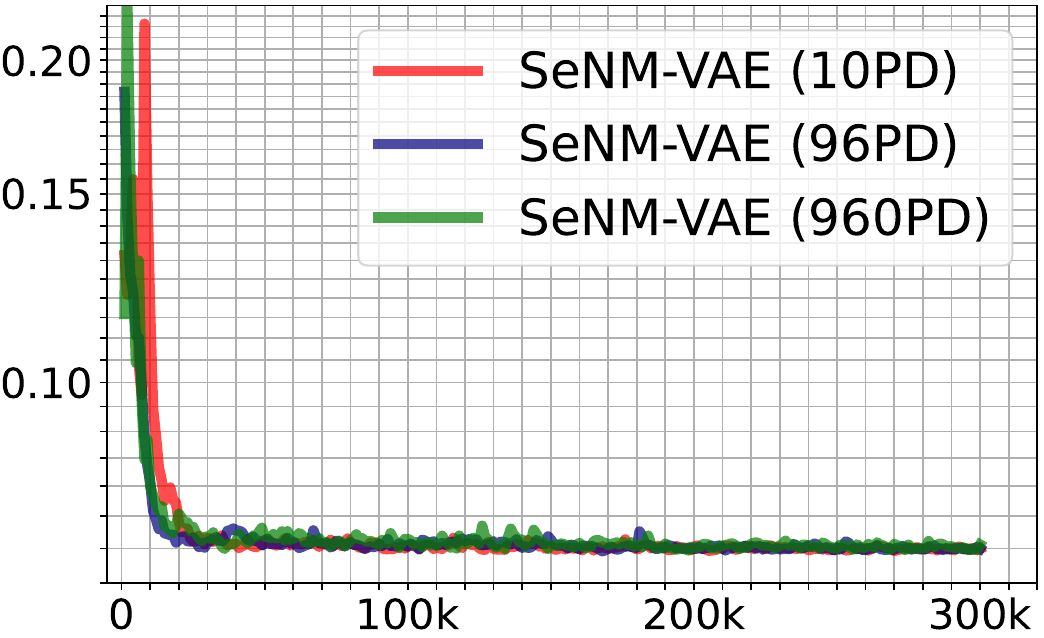}
    \caption{Loss curve of SeNM-VAE during training. Our model converges to the minimum steadily and uniformly, regardless of the quantity of paired samples utilized.}
    \label{fig:loss_curve}
\end{figure}
% \begin{equation}
%     \mathbb E_{q\left(\rvz | \rvx, \rvy\right)} \log p\left(\rvx | \rvz\right)
% \end{equation}

\subsection{Degradation modeling in real-world SR}

% As a supplement to the noise synthesis experiment in the main paper, we perform similar tests for the real-world SR task. In this experiment, we compared our semi-supervised method with two competitive unpaired degradation modeling methods, namely DeFlow and LUD-VAE. All methods were trained on AIM19 and NTIRE20 training sets respectively, with an additional 10 paired samples from the validation dataset used to train SeNM-VAE. To evaluate the degradation modeling capacity, we trained the ESRGAN model using paired data generated from the comparison methods. The resulting metrics, namely PSNR, SSIM, and LPIPS, on the AIM19 and NTIRE20 datasets are reported in Table~\ref{tab:aim19} and Table~\ref{tab:ntire20}, respectively. The results demonstrate that our semi-supervised method can effectively learn the degradation model in real-world SR.

As a complementary investigation to the noise synthesis experiment presented in the main paper, we conduct analogous assessments to evaluate the quality of the generated training pairs in real-world SR tasks. The configuration for training the degradation modeling methods remains consistent with that outlined in the downstream SR experiment. Subsequently, we train the ESRGAN~\cite{wang2018esrgan} model using paired data derived from the comparison methods. The resultant metrics, including PSNR, SSIM, and LPIPS, on both the AIM19 and NTIRE20 datasets, are detailed in Tables~\ref{tab:aim19} and Table~\ref{tab:ntire20}, respectively. These results demonstrate the effectiveness of our semi-supervised approach in learning the degradation model in real-world SR scenarios.

\begin{table}
	\setlength{\tabcolsep}{7.3pt}
	\renewcommand{\arraystretch}{0.8}
	\centering	
	\begin{tabular}{cccc}
		\hline
		\specialrule{0em}{.9pt}{.9pt}
		\rowcolor[gray]{0.95} 
		Method      & PSNR $\uparrow$ & SSIM $\uparrow$ & LPIPS $\downarrow$ \\ 
		\specialrule{0em}{.9pt}{.9pt}
		\hline
        FSSR   & 20.97 & 0.5383 & 0.374 \\
	Impressionism    & 21.93 & 0.6128 & 0.426 \\
        DASR     & 21.05 & 0.5674 & 0.376 \\
        DeFlow        & 21.43 & 0.6003 & 0.349 \\
        LUD-VAE       & 22.25 & 0.6194 & 0.341 \\
        \hline
        \textbf{SeNM-VAE}                   & \textbf{22.37} & \textbf{0.6307} & \textbf{0.335} \\
        \hline
	\end{tabular}
	\caption{Comparison of SR performance on AIM19. SeNM-VAE is trained with 10 paired data.}
    \label{tab:aim19}
\end{table}

\begin{table}
	\setlength{\tabcolsep}{7.3pt}
	\renewcommand{\arraystretch}{0.8}
	\centering	
	\begin{tabular}{cccc}
		\hline
		\specialrule{0em}{.9pt}{.9pt}
		\rowcolor[gray]{0.95} 
		Method      & PSNR $\uparrow$ & SSIM $\uparrow$ & LPIPS $\downarrow$ \\ 
		\specialrule{0em}{.9pt}{.9pt}
		\hline
        FSSR   & 21.01 & 0.4229 & 0.435 \\
	Impressionism     & 25.24 & 0.6740 & 0.230 \\
        DASR     & 22.98 & 0.5093 & 0.379 \\
        DeFlow        & 24.95 & 0.6746 & 0.217 \\
        LUD-VAE       & 25.78 & 0.7196 & 0.220 \\
        \hline
        \textbf{SeNM-VAE}                   & \textbf{25.91} & \textbf{0.7212} & \textbf{0.216} \\
        \hline
	\end{tabular}
	\caption{Comparison of SR performance on NTIRE20. SeNM-VAE is trained with 10 paired data.}
    \label{tab:ntire20}
\end{table}

\subsection{Effects of varying mixture weights}

In our main paper, we define the inference model $q\left(\rvz | \rvx, \rvy\right)$ as a linear combination of two mixture components $q\left(\rvz | \rvx\right)$ and $q\left(\rvz | \rvy\right)$, expressed as:
\begin{equation*}
    q\left(\rvz | \rvx\right) = p_1 q\left(\rvz | \rvx\right) + p_2 q\left(\rvz | \rvy\right),
\end{equation*}
where $p_1$ and $p_2$ are mixture weights. In this experiment, we investigate the impact of different $p_1$ and $p_2$ values. Given that $p_2 = 1 - p_1$, we evaluate five cases for $p_1$ using the SIDD dataset, each with 10 paired samples. As shown in Table~\ref{tab:optim_p_1}, the noisy data generated by SeNM-VAE achieves the minimum FID and KLD values when $p_1 = 0.5$, while the downstream denoising network (DnCNN~\cite{zhang2017beyond}) exhibits its highest PSNR when $p_1 = 0.7$.

\begin{table}[ht]
    % \vspace{-8pt}
    \setlength{\tabcolsep}{7.5pt}
    \renewcommand{\arraystretch}{0.8}
    \centering	
    \begin{tabular}{cccccc}
        \hline
        \specialrule{0em}{.9pt}{.9pt}
        \rowcolor[gray]{0.95} 
        $p_1$ & 0.1 & 0.3 & 0.5 & 0.7 & 0.9 \\ 
        \specialrule{0em}{.9pt}{.9pt}
        \hline
        FID $\downarrow$ & 17.39 & 18.27 & \textbf{17.25} & 19.20 & 17.99 \\
        KLD $\downarrow$ & 0.037 & 0.044 & \textbf{0.036} & 0.039 & 0.044 \\
        PSNR $\uparrow$ & 36.48 & 36.28 & 36.73 & \textbf{36.98} & 36.72 \\
        \hline
    \end{tabular}
    \caption{Comparison of noise quality on SIDD validation dataset. DnCNN~\cite{zhang2017beyond} is used as a downstream denoising model.}
    \label{tab:optim_p_1}
    % \vspace{-14pt}
\end{table}

\section{Visual results}

Owing to the space constraints within the main context, we exhibit additional visualizations of synthetic noise, real-world denoising results, and real-world SR results as a supplement.

\subsection{Noise synthesis}

We present synthesized noisy images generated by SeNM-VAE trained with varying numbers of paired data. Furthermore, we conduct a comparative analysis with fully-paired deep noise modeling methods, including DANet~\cite{yue2020dual}, Flow-sRGB~\cite{kousha2022modeling}, and NeCA-W~\cite{fu2023srgb}. The visual results on the SIDD validation dataset are depicted in Figure~\ref{supp_noise-generation}.

\subsection{Real-world denoising}

{\noindent \bf Downstream denoising.} We employ DRUNet~\cite{zhang2021plug} as the downstream denoising model and train it on the paired domain alongside synthetic paired samples generated by SeNM-VAE. We compare our semi-supervised denoising method with direct training on the paired domain and several self-supervised denoising methods, namely CVF-SID~\cite{neshatavar2022cvf}, AP-BSN + R$^3$~\cite{lee2022ap}, SCPGabNet~\cite{lin2023unsupervised}, and SDAP(S)(E)~\cite{pan2023random}. Denoising results on the SIDD validation dataset are displayed in Figure~\ref{supp_downstream-denoising1}, Figure~\ref{supp_downstream-denoising2}, and Figure~\ref{supp_downstream-denoising3}.

{\noindent \bf Finetune denoising.} We perform fine-tuning on NAFNet~\cite{chen2022simple}, a pre-trained denoising model, using additional training samples generated by SeNM-VAE trained with full paired data from the SIDD training dataset. The finetuned NAFNet is compared against its original version as well as three alternative methods, namely Uformer~\cite{wang2022uformer}, MAXIM~\cite{tu2022maxim}, and Restormer~\cite{zamir2022restormer}. Denoising results on the SIDD validation dataset are presented in Figure~\ref{supp_finetune-denoising}.

\subsection{Real-world SR}

SeNM-VAE is also employed to simulate the degradation process of real-world SR tasks. We leverage ESRGAN~\cite{wang2018esrgan} as the downstream model. Our semi-supervised SR method is compared with a supervisedly trained ESRGAN, along with five unpaired degradation modeling methods, namely FSSR~\cite{fritsche2019frequency}, Impressionism~\cite{ji2020real}, DASR~\cite{wei2021unsupervised}, DeFlow~\cite{wolf2021deflow}, and LUD-VAE~\cite{zheng2022learn}. Evaluation is conducted on the AIM19 and NTIRE20 validation datasets. Visualizations of the SR results are provided in Figure~\ref{supp_SR_AIM1}, Figure~\ref{supp_SR_AIM2}, Figure~\ref{supp_SR_NTIRE1}, and Figure~\ref{supp_SR_NTIRE2}.

\begin{figure*}
    \centering
    \includegraphics[width = 1\linewidth]{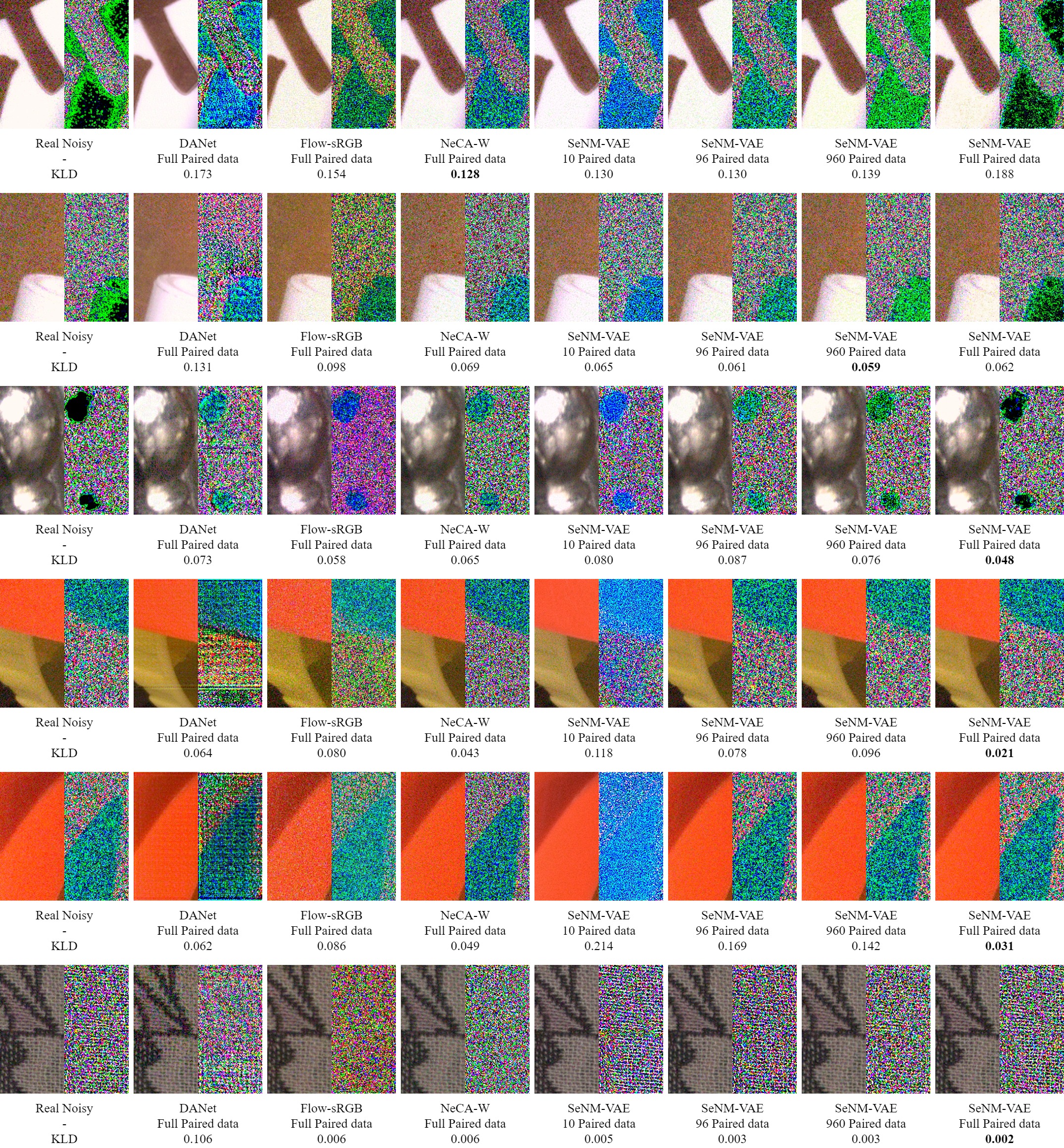}
    \caption{Visual comparisons of noise generation on the SIDD validation set. KLD value is reported as the performance metric.}
    \label{supp_noise-generation}
\end{figure*}

\begin{figure*}
    \centering
    \includegraphics[width = 1\linewidth]{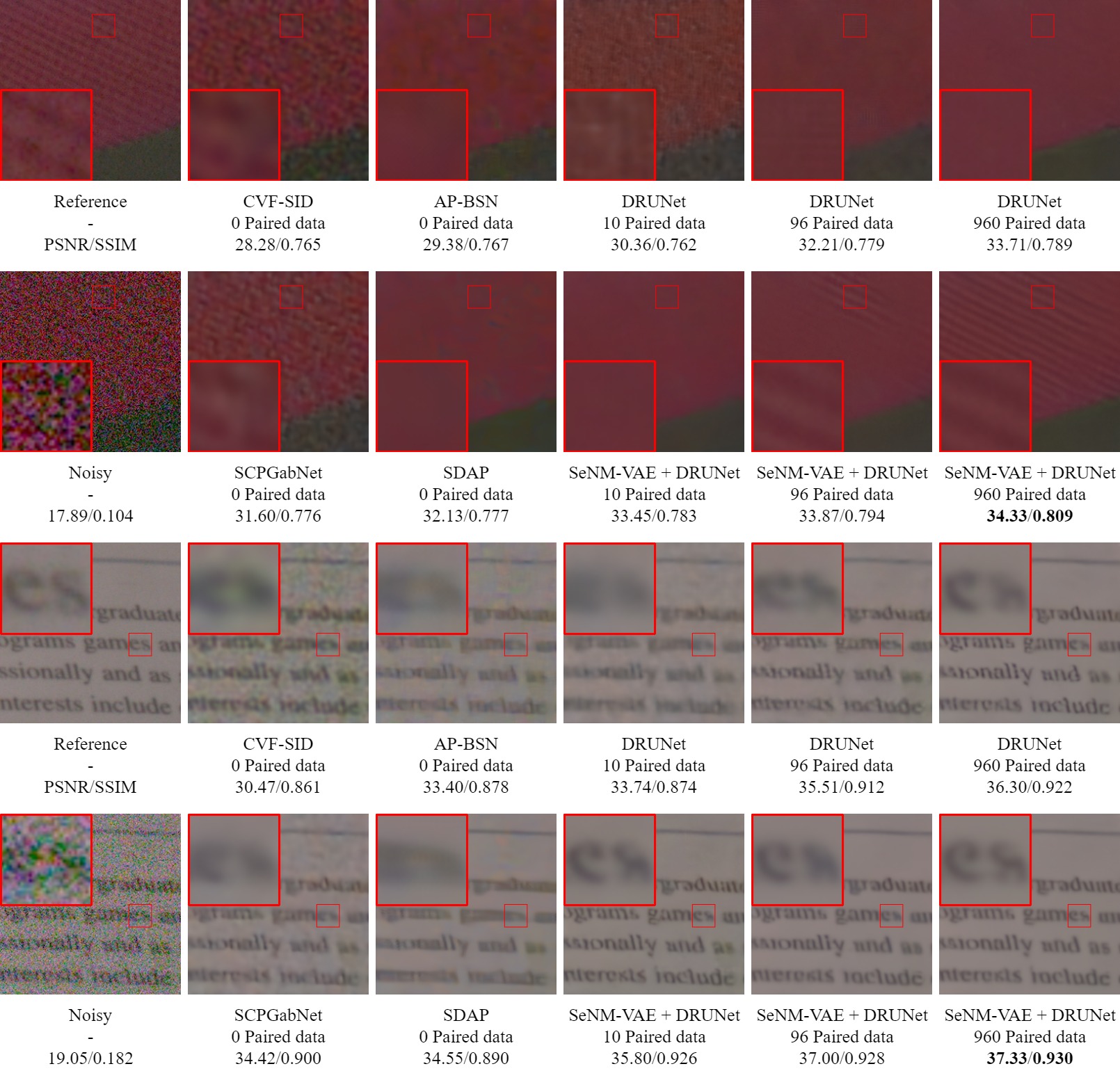}
    \caption{Visual comparisons of downstream denoising results on the SIDD validation set. Performance metrics, including PSNR and SSIM values, are reported for evaluation.}
    \label{supp_downstream-denoising1}
\end{figure*}

\begin{figure*}
    \centering
    \includegraphics[width = 1\linewidth]{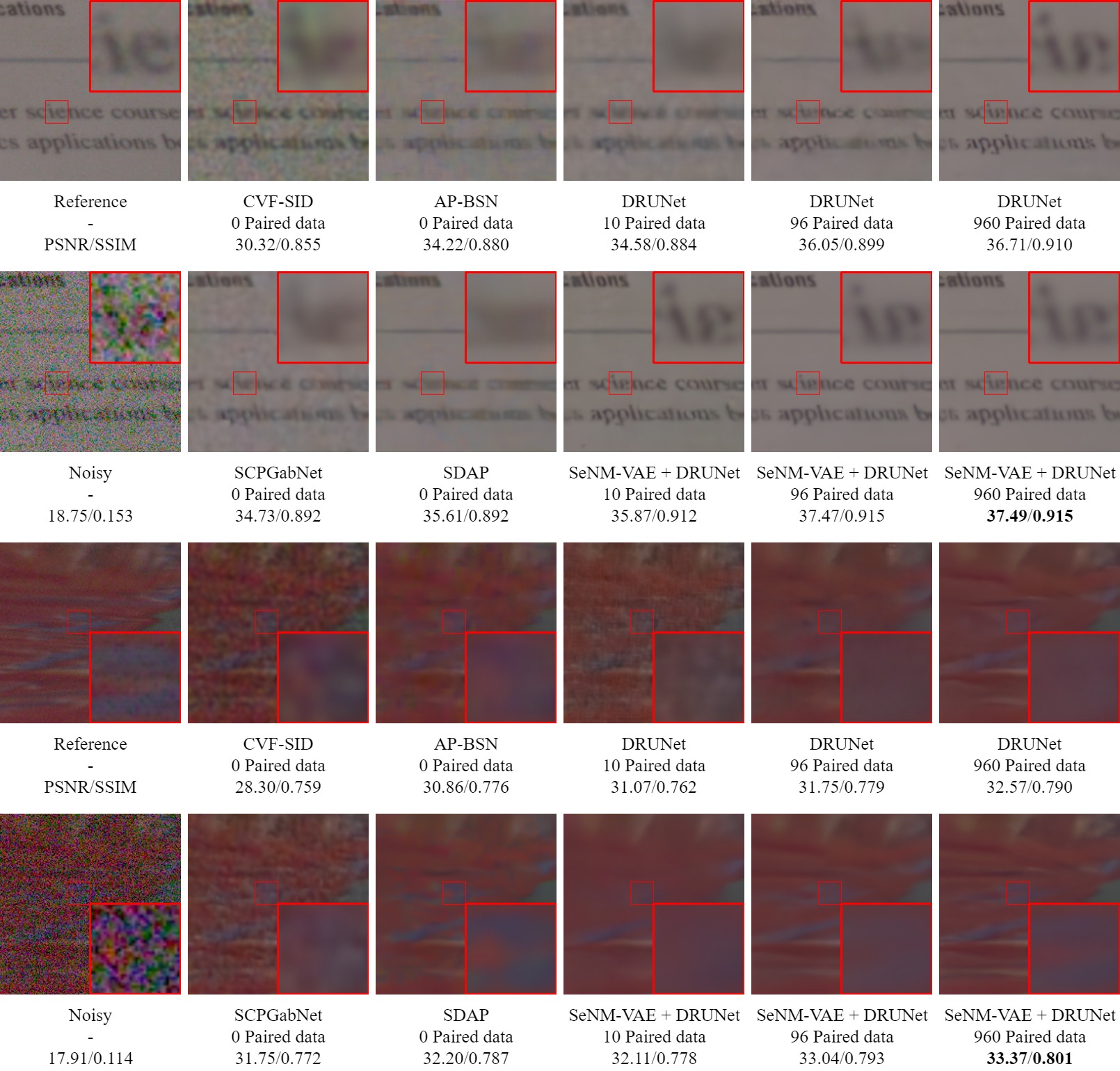}
    \caption{Visual comparisons of downstream denoising results on the SIDD validation set. Performance metrics, including PSNR and SSIM values, are reported for evaluation.}
    \label{supp_downstream-denoising2}
\end{figure*}

\begin{figure*}
    \centering
    \includegraphics[width = 1\linewidth]{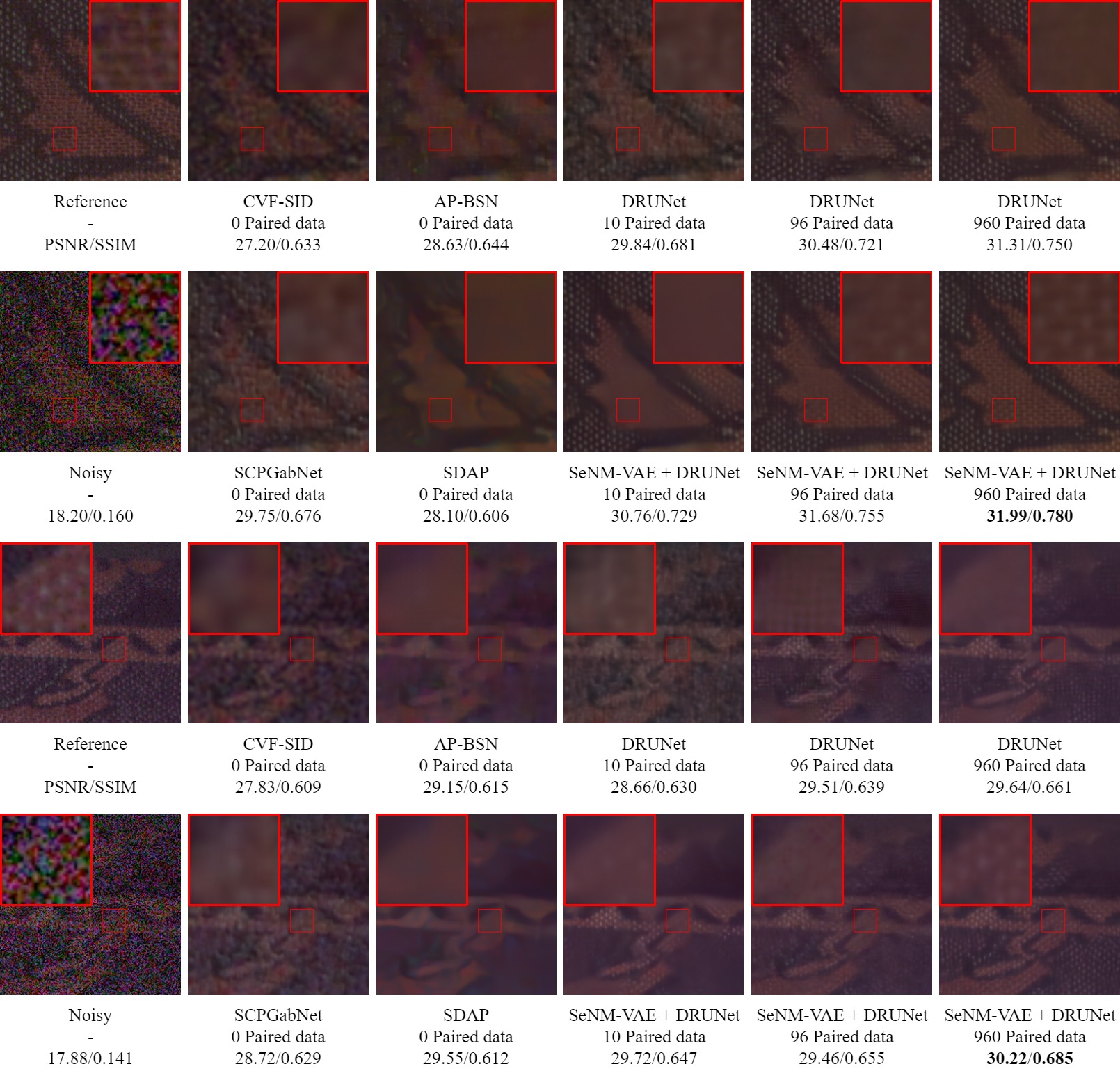}
    \caption{Visual comparisons of downstream denoising results on the SIDD validation set. Performance metrics, including PSNR and SSIM values, are reported for evaluation.}
    \label{supp_downstream-denoising3}
\end{figure*}

\begin{figure*}
    \centering
    \includegraphics[width = 1\linewidth]{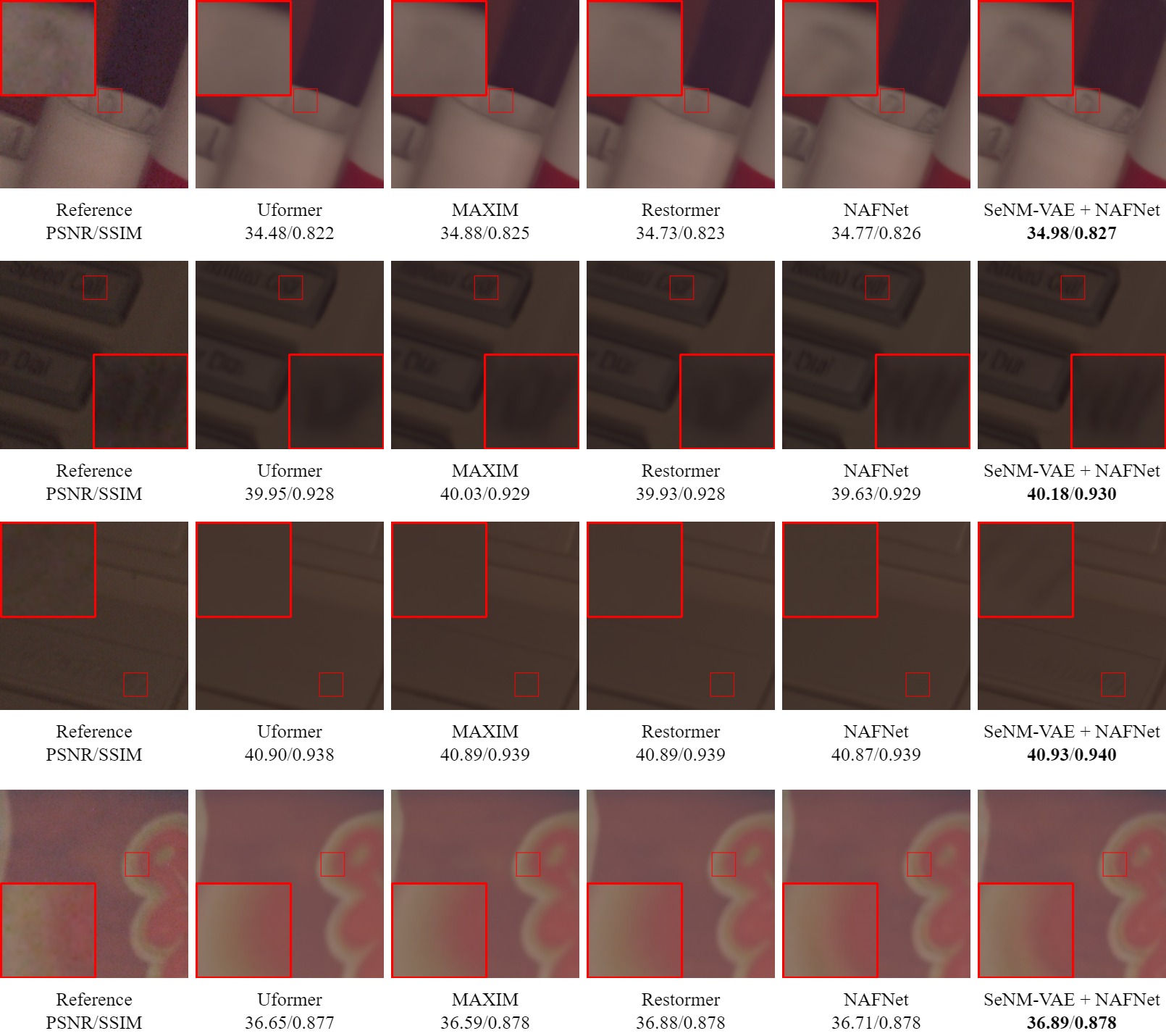}
    \caption{Visual comparisons of fine-tuned denoising results on the SIDD validation set. Performance metrics, including PSNR and SSIM values, are reported for evaluation.}
    \label{supp_finetune-denoising}
\end{figure*}

\clearpage

\begin{figure*}
    \centering
    \includegraphics[width = 1\linewidth]{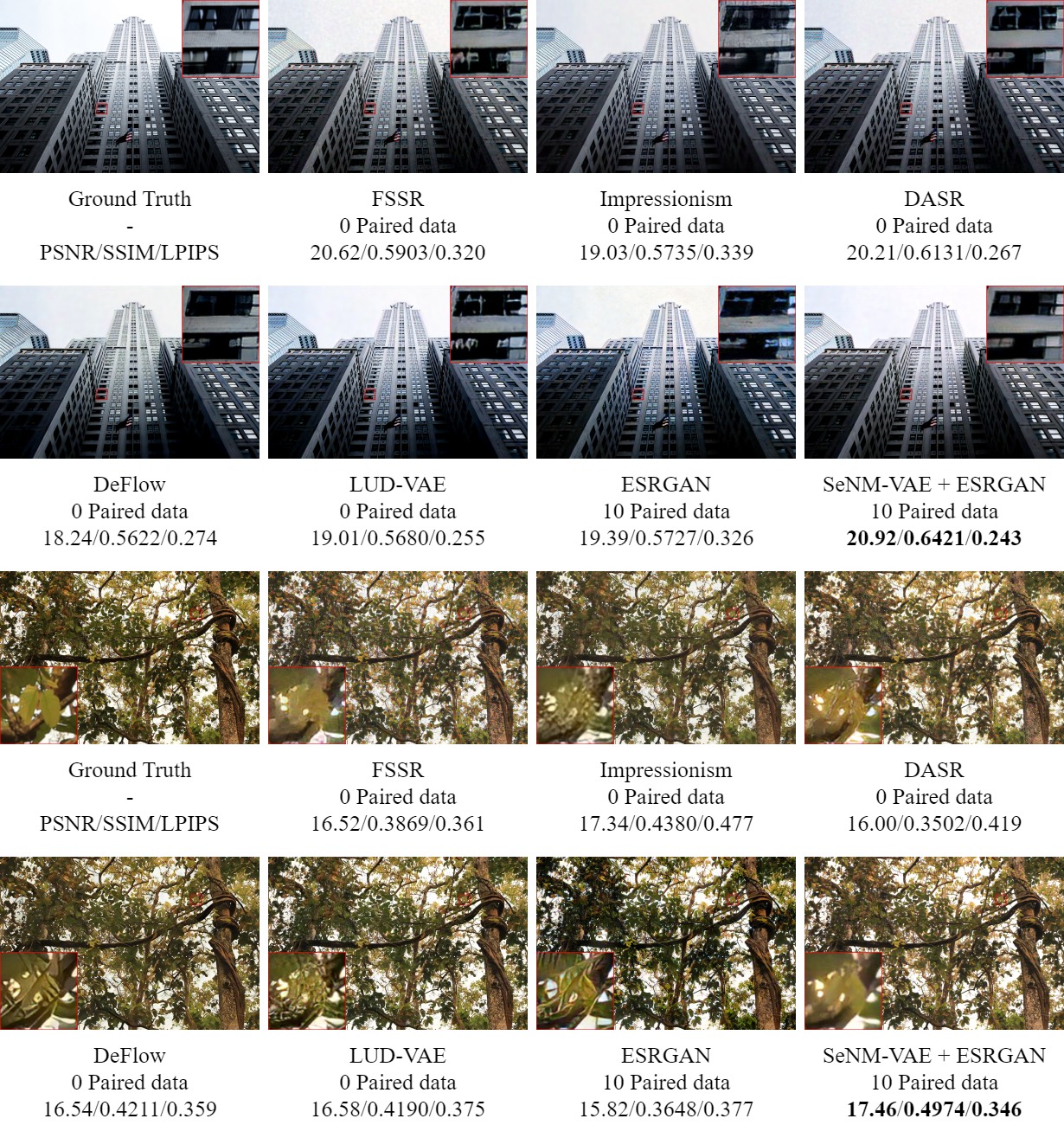}
    \caption{Visual comparisons of real-world SR results on the AIM19 validation set. Performance metrics, including PSNR, SSIM, and LPIPS values, are provided for evaluation.}
    \label{supp_SR_AIM1}
\end{figure*}

\clearpage

\begin{figure*}
    \centering
    \includegraphics[width = 1\linewidth]{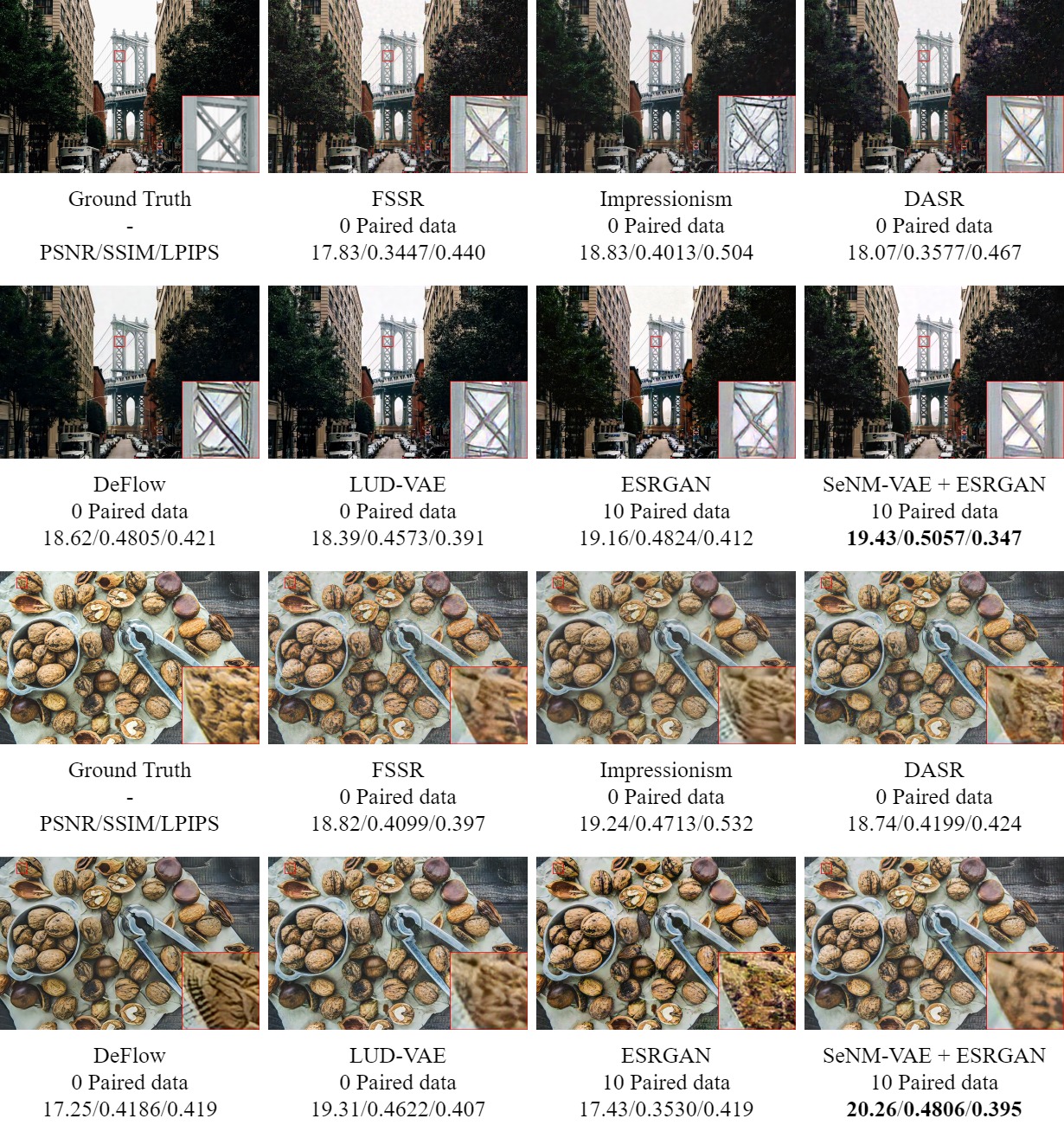}
    \caption{Visual comparisons of real-world SR results on the AIM19 validation set. Performance metrics, including PSNR, SSIM, and LPIPS values, are provided for evaluation.}
    \label{supp_SR_AIM2}
\end{figure*}

\clearpage

\begin{figure*}
    \centering
    \includegraphics[width = 1\linewidth]{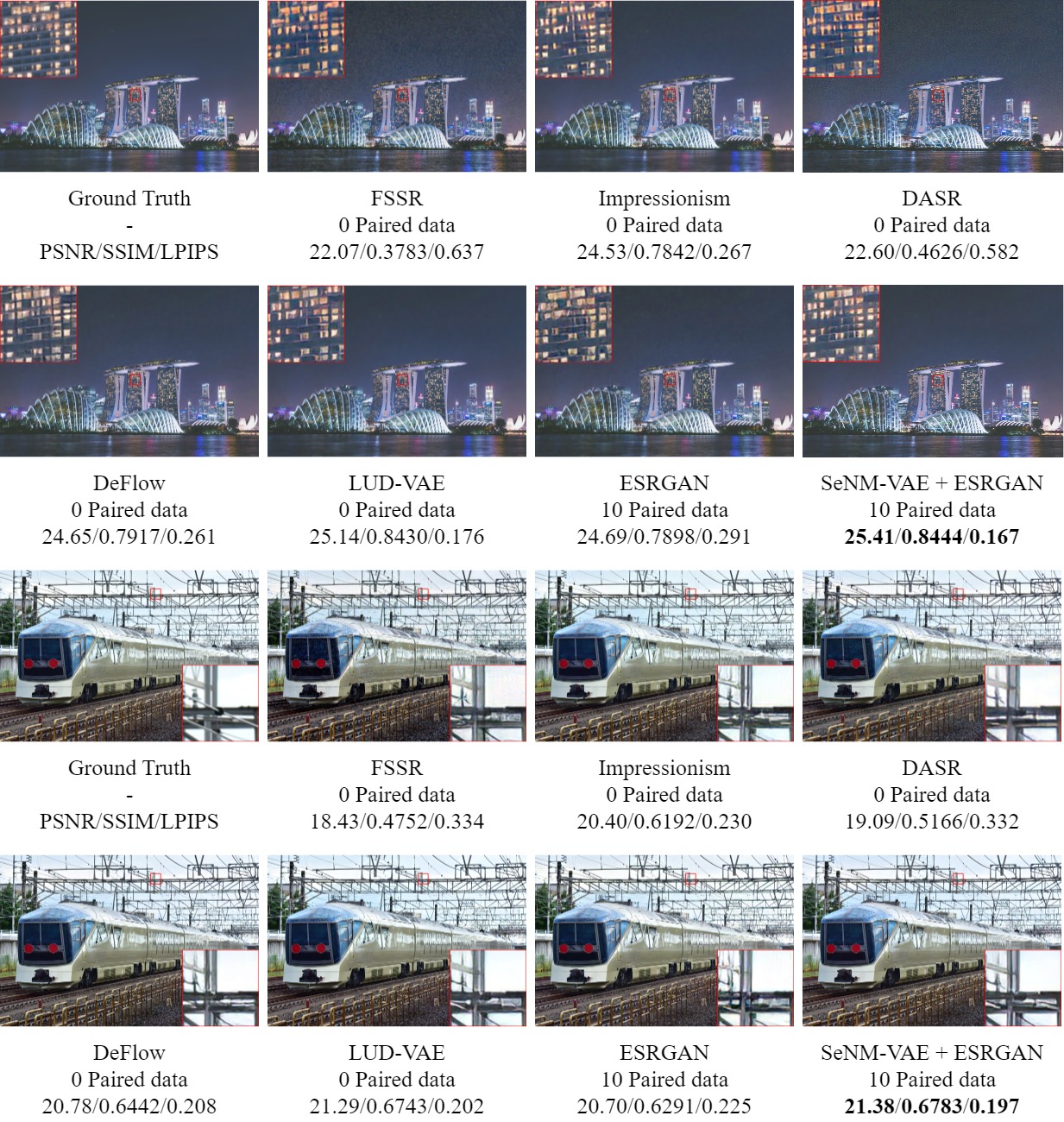}
    \caption{Visual comparisons of real-world SR results on the NTIRE20 validation set. Performance metrics, including PSNR, SSIM, and LPIPS values, are provided for evaluation.}
    \label{supp_SR_NTIRE1}
\end{figure*}

\clearpage

\begin{figure*}
    \centering
    \includegraphics[width = 1\linewidth]{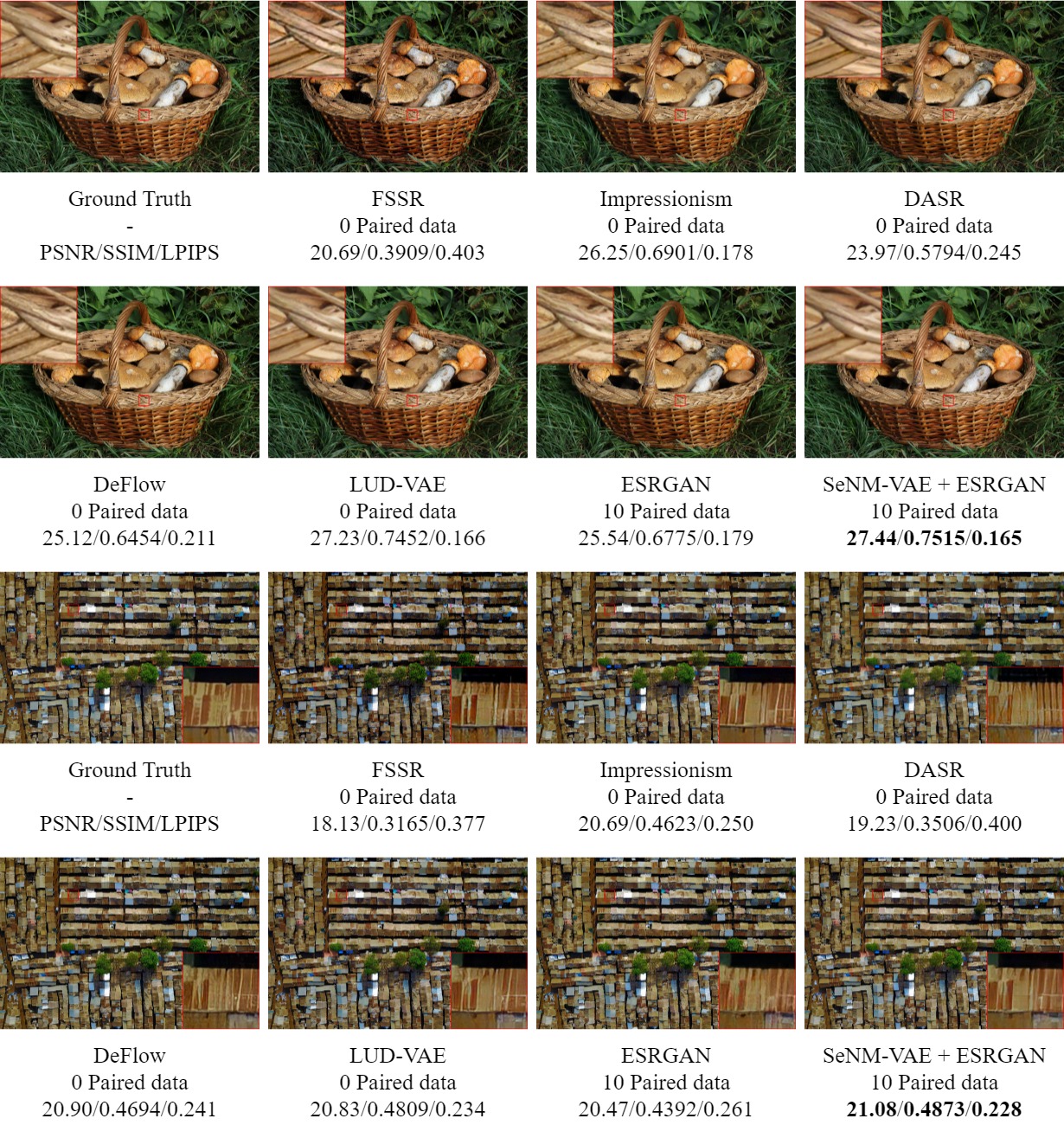}
    \caption{Visual comparisons of real-world SR results on the NTIRE20 validation set. Performance metrics, including PSNR, SSIM, and LPIPS values, are provided for evaluation.}
    \label{supp_SR_NTIRE2}
\end{figure*}

\clearpage

% \input{sec/X_suppl}
% {
%     \small
%     \bibliographystyle{ieeenat_fullname}
%     \bibliography{main}
% }

\end{document}